\documentclass{article}

\usepackage{amsmath,amsfonts,bm}

\def\eqref#1{equation~\ref{#1}}

\def\1{\bm{1}}

\DeclareMathAlphabet{\mathsfit}{\encodingdefault}{\sfdefault}{m}{sl}
\SetMathAlphabet{\mathsfit}{bold}{\encodingdefault}{\sfdefault}{bx}{n}

\def\gD{{\mathcal{D}}}

\def\gR{{\mathcal{R}}}

\def\gX{{\mathcal{X}}}
\def\gY{{\mathcal{Y}}}

\newcommand{\ptrain}{P_{\rm{train}}}

\newcommand{\E}{\mathbb{E}}

\newcommand{\R}{\mathbb{R}}

\newcommand{\Var}{\mathrm{Var}}

\DeclareMathOperator*{\argmin}{arg\,min}

\usepackage{microtype}
\usepackage{graphicx}
\usepackage{float, caption, subcaption}
\usepackage{booktabs} 
\usepackage{multirow}
\usepackage{amsthm}
\usepackage{xcolor}

\usepackage{hyperref}
\usepackage{xr}
\makeatletter
\newcommand*{\addFileDependency}[1]{
  \typeout{(#1)}
  \@addtofilelist{#1}
  \IfFileExists{#1}{}{\typeout{No file #1.}}
}
\makeatother

\usepackage{zref-base}
\usepackage{atveryend}
\makeatletter
\AfterLastShipout{%
  \if@filesw   
    \begingroup
      \zref@setcurrent{default}{\the\value{equation}}%
      \zref@wrapper@immediate{%
        \zref@labelbyprops{eq-last}{default}%
      }%
    \endgroup
    \begingroup
      \zref@setcurrent{default}{\the\value{thm}}%
      \zref@wrapper@immediate{%
        \zref@labelbyprops{thm-last}{default}%
      }%
    \endgroup
  \fi
}
\makeatother

\newcommand{\cvar}{\textnormal{CVaR}}
\newcommand{\tv}{\textnormal{TV}}

\newtheorem{thm}{Theorem}
\newtheorem{prop}[thm]{Proposition}
\newtheorem{lem}[thm]{Lemma}
\newtheorem{cor}[thm]{Corollary}

\usepackage[accepted]{icml2021}

\newcommand{\camready}[1]{#1}

\icmltitlerunning{DORO: Distributional and Outlier Robust Optimization}

\begin{document}

\twocolumn[
\icmltitle{DORO: Distributional and Outlier Robust Optimization}



\icmlsetsymbol{equal}{*}

\begin{icmlauthorlist}
\icmlauthor{Runtian Zhai}{equal,cmu}
\icmlauthor{Chen Dan}{equal,cmu}
\icmlauthor{J. Zico Kolter}{cmu}
\icmlauthor{Pradeep Ravikumar}{cmu}
\end{icmlauthorlist}

\icmlaffiliation{cmu}{School of Computer Science, Carnegie Mellon University, Pittsburgh, Pennsylvania, USA}

\icmlcorrespondingauthor{Runtian Zhai}{rzhai@cmu.edu}

\icmlkeywords{Machine Learning, ICML, Distributionally Robust Optimization, Robust Statistics, Algorithmic Fairness, Class Imbalance}

\vskip 0.3in
]



\printAffiliationsAndNotice{\icmlEqualContribution} 

\begin{abstract}
Many machine learning tasks involve subpopulation shift where the testing data distribution is a subpopulation of the training distribution. For such settings, a line of recent work has proposed the use of a variant of empirical risk minimization(ERM) known as distributionally robust optimization (DRO). In this work, we apply DRO to real, large-scale tasks with subpopulation shift, and observe that DRO performs relatively poorly, and moreover has severe instability. We identify one direct cause of this phenomenon: sensitivity of DRO to outliers in the datasets. To resolve this issue, we propose the framework of DORO, for Distributional and Outlier Robust Optimization. At the core of this approach is a refined risk function which prevents DRO from overfitting to potential outliers. We instantiate DORO for the Cressie-Read family of R\'enyi divergence, and delve into two specific instances of this family: CVaR and $\chi^2$-DRO. We theoretically prove the effectiveness of the proposed method, and empirically show that DORO improves the performance and stability of DRO with experiments on large modern datasets, thereby positively addressing the open question raised by \cite{pmlr-v80-hashimoto18a}. Codes are available at \url{https://github.com/RuntianZ/doro}.
\end{abstract}

\section{Introduction}

Many machine learning tasks require models to perform well under distributional shift, where the training and the testing data distributions are different. One type of distributional shift that arouses great research interest is \textit{subpopulation shift}, where the testing distribution is a specific or the worst-case subpopulation of the training distribution. A wide range of tasks can be modeled as subpopulation shift problems, such as learning for algorithmic fairness \cite{dwork2012fairness,barocas2016big} where we want to test model's performance on key demographic subpopulations, and learning with class imbalance \cite{japkowicz2000class,galar2011review} where we train a classifier on an imbalanced dataset with some minority classes having much fewer samples than the others, and we want to maximize the classifier's accuracy on the minority classes instead of its overall average accuracy.

Distributionally robust optimization (DRO) \cite{namkoong2016stochastic,duchi2018learning} refers to a family of learning algorithms that minimize the model's loss over the worst-case distribution in a neighborhood of the observed training distribution. Generally speaking, DRO trains the model on the worst-off subpopulation, and when the subpopulation membership is unknown, it focuses on the worst-off training instances, that is, the tail performance of the model. Previous work has shown effectiveness of DRO in subpopulation shift settings, such as algorithmic fairness \cite{pmlr-v80-hashimoto18a} and class imbalance \cite{xu2020class}. 

However, in our empirical investigations, when we apply DRO to real tasks on modern datasets, we observe that DRO suffers from poor performance and severe instability during training. \camready{The issue that DRO is sensitive to outliers has been raised by several previous papers \cite{pmlr-v80-hashimoto18a,hu2018does,zhu2020generalized} .} In this paper, we study the cause of these problems with DRO, and develop approaches to address them.
%

%
In particular, we identify and study one key factor that we find directly leads to DRO's sub-optimal behavior: DRO's sensitivity to outliers that widely exist in modern datasets. In general, DRO maximizes a model's tail performance by putting more weights on the ``harder'' instances, i.e. those which incur higher losses during training. On the one hand, this allows DRO to focus its attention on worst-off sub-populations. But on the other hand, since outliers are intuitively ``hard'' instances that incur higher losses than inliers, DRO is prone to assign large weights to outliers, resulting in both a drop in performance, and training instability. To provide empirical insights into how outliers affect DRO, in Section \ref{sec:dro-not-robust} we conducted experiments examining how the performance of DRO changes as we removed or added outliers to the dataset. The results of these experiments indicate that outliers bring about the observed bad performance of DRO. Thus, it is crucial to first enhance the robustness of DRO to outliers before applying it to real-world applications.

To this end, we propose DORO, an outlier robust refinement of DRO which takes inspiration from robust statistics. At the core of this approach is a refined risk function which prevents DRO from overfitting to potential outliers. Intuitively speaking, the new risk function adaptively filters out a small fraction of data with high risk during training, which is potentially caused by outliers. 
Figure \ref{fig:illust} illustrates the difference between DRO and DORO. In Section \ref{sec:method} we implement DORO for the Cressie-Read family of R\'enyi divergence, and for our theoretical and empirical study we primarily focus on CVaR-DORO and $\chi^2$-DORO. In Section \ref{sec:theory} we provide theoretical results guaranteeing that DORO can effectively handle subpopulation shift in the presence of outliers. Then, in Section \ref{sec:experiments} we empirically demonstrate that DORO improves the performance and stability of DRO. We conduct large-scale experiments on three datasets: the tabular dataset COMPAS, the vision dataset CelebA, and the language dataset CivilComments-Wilds.

\begin{figure}[!t]
    \centering
    \includegraphics[width=.85\linewidth]{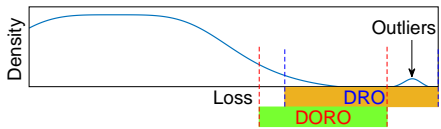}
    \caption{DORO avoids overfitting to outliers.}
    \label{fig:illust}
    \vskip -0.1in
\end{figure}

\paragraph{Contributions}
Our contributions are summarized below:
\begin{itemize}
    \item We demonstrate that the sensitivity of DRO to outliers is a direct cause of the irregular behavior of DRO with some intriguing experimental results in Section \ref{sec:dro-not-robust}.
    \item We propose and implement DORO as an outlier robust refinement of DRO in Section \ref{sec:method}. Then, in Section \ref{sec:theory} we provide theoretical guarantees for DORO.
    \item We conduct large-scale experiments in Section \ref{sec:experiments} and empirically show that DORO improves the performance and stability of DRO. We also analyze the effect of hyperparameters on DRO and DORO.
\end{itemize}

\paragraph{Related Work}
Distributional shift naturally arises in many machine learning applications and has been widely studied in statistics, applied probability and optimization \cite{shimodaira2000improving,huang2006correcting,bickel2007discriminative,quionero2009dataset}. One common type of distributional shift is \textit{domain generalization} where the training and testing distributions consist of distinct domains, and relevant topics include domain adaptation \cite{patel2015visual,wang2018deep} and transfer learning \cite{pan2009survey,tan2018survey}. Another common type of distributional shift studied in this paper is subpopulation shift, where the two distributions consist of the same group of domains. Subpopulation shift is closely related to algorithmic fairness and class imbalance. For algorithmic fairness, a number of fairness notions have been proposed, such as individual fairness \cite{dwork2012fairness,zemel2013learning}, group fairness \cite{hardt2016equality,zafar2017fairness}, counterfactual fairness \cite{kusner2017counterfactual} and Rawlsian Max-Min fairness \cite{rawls2001justice,pmlr-v80-hashimoto18a}. The setting of subpopulation shift is most closely related to the Rawlsian Max-Min fairness notion. Several recent papers \cite{pmlr-v80-hashimoto18a,oren-etal-2019-distributionally,xu2020class} proposed using DRO to deal with subpopulation shift, but it was also observed that DRO was prone to overfit in practice \cite{sagawa2019distributionally,sagawa2020investigation}. \cite{pmlr-v80-hashimoto18a} raised the open question whether it is possible to design algorithms both fair to unknown latent subpopulations and robust to outliers, and this work answers this question positively.

Outlier robust estimation is a classic problem in statistics starting with the pioneering works of \cite{tukey1960survey,huber1992robust}. Recent works in statistics and machine learning \cite{lai2016agnostic,diakonikolas2017being,prasad2018robust,diakonikolas2019robust} provided efficiently computable outlier-robust estimators for high-dimensional mean estimation with corresponding error guarantees. 
Outliers have a greater effect on the performance of DRO than ERM \cite{hu2018does}, due to its focus on the tail performance, so removing this negative impact of outliers is crucial for the success of DRO in its real-world applications. One closely related recent work is \cite{NEURIPS2020_9f60ab2b}, and DORO can be viewed as a combination of risk-averse and risk-seeking methods discussed in this paper.

\section{Background}
\label{sec:preliminary}

This section provides the necessary background of subpopulation shift and DRO.

\subsection{Subpopulation Shift}
\label{sec:group-fairness}

A machine learning task with subpopulation shift requires a model that performs well on the data distribution of each subpopulation. Let the input space be $\gX$ and the label space be $\gY$. We are given a training set containing $m$ samples i.i.d. sampled from some data distribution $P$ over $\gX \times \gY$. There are $K$ predefined domains (subpopulations) $\gD_1,\cdots,\gD_K$, each of which is a subset of $\gX \times \gY$. For example, in an algorithmic fairness task, domains are demographic groups defined by a number of \textit{protected features} such as race and sex. Let $P_k(z) = P(z|z \in \gD_k)$ be the conditional training distribution over $\gD_k$, where $z=(x,y)$. The goal is to train a model $f_\theta: \gX \rightarrow \gY$ parameterized by $\theta \in \Theta$ that performs well over every $P_k$. Denote the expected risk over $P$ by $\gR(\theta;P)=\E_{Z \sim P}[\ell(\theta;Z)]$ where $\ell(\theta;z)$ is a measurable loss function. Then the expected risk over $P_k$ is $\gR_k(\theta;P)=\E_{Z\sim P_k}[\ell(\theta;Z)]$. The objective is to minimize the \textit{worst-case risk} defined as
\begin{equation}
\label{eqn:rmax}
\gR_{\max}(\theta;P) = \max_{k=1,\cdots,K} \gR_k(\theta;P)
\end{equation}

Several different settings were studied by previous work:

\paragraph{Overlapping vs Non-overlapping}
The \textit{overlapping} setting allows the domains to overlap with each other while \textit{non-overlapping} does not. For example, suppose we have two protected features: race (White and Others) and sex (Male and Female). Under either setting we will have four domains. Under the overlapping setting we will have \textit{White}, \textit{Others}, \textit{Male} and \textit{Female}, while under the non-overlapping setting we will have \textit{White Male}, \textit{White Female}, \textit{Others Male} and \textit{Others Female}. All the experiments in this work are conducted under the overlapping setting. Each instance may belong to zero, one or more domains.

\paragraph{Domain-Aware vs Domain-Oblivious}
Some previous work has assumed that domain memberships of instances are known at least during training. This is called the \textit{domain-aware} setting. However, \cite{pmlr-v80-hashimoto18a} argue that in many real applications, domain memberships are unknown during training, either because it is hard to extract the domain information from the input, or because it is hard to identify all protected features. Thus, a line of recent work \cite{pmlr-v80-hashimoto18a,lahoti2020fairness} studies the \textit{domain-oblivious} setting, in which the training algorithm does not know the domain membership of any instance (even the number of domains $K$ is unknown). In this work, we focus on the domain-oblivious setting.

\subsection{Distributionally Robust Optimization (DRO)}

Under the domain-oblivious setting, we cannot compute the worst-case risk since we have no access to $\gD_1,\cdots,\gD_K$. In this case, the framework of DRO instead maximizes the performance over the worst-off subpopulation in general. Specifically, given some divergence $D$ between distributions, DRO aims to minimize the expected risk over the worst-case distribution $Q$ (that is absolutely continuous with respect to training distribution $P$, so that $Q \ll P$) in a ball w.r.t. divergence $D$ around the training distribution $P$.

%
Thus, while empirical risk minimization (ERM) algorithm minimizes the expected risk $\gR(\theta;P)$, DRO minimizes the \textit{expected DRO risk} defined as:
\begin{equation}
\label{eqn:dro-risk}
\gR_{D,\rho}(\theta;P) = \sup_{Q \ll P}\{ \E_Q [\ell(\theta;Z)]: D(Q \parallel P) \leq \rho \}
\end{equation}
for some $\rho > 0$.
Different divergence functions $D$ derive different DRO risks. In this work, we focus on the Cressie-Read family of R\'enyi divergence \cite{cressie1984multinomial} formulated as:
\begin{equation}
\label{eqn:cressie-read}
	D_\beta(Q \parallel P) =\int  f_\beta(\frac{dQ}{dP}) dP
\end{equation}
where $\beta >1$, and $f_\beta(t)$ is defined as:
\begin{equation}
	f_\beta(t) = \frac{1}{\beta(\beta-1)}\left( t^\beta - \beta t + \beta -1\right)
\end{equation}

An advantage of the Cressie-Read family is that it has the following convenient dual characterization (see Lemma 1 of \cite{duchi2018learning} for the proof):
	\begin{equation}\label{eqn:dro_dual}
		 \gR_{D_{\beta},  \rho}(\theta; P) = \inf_{\eta \in \R} \left\lbrace c_\beta(\rho) \E_{P}[(l(\theta;Z)-\eta)_+^{\beta_*}]^{\frac{1}{\beta_*} } + \eta\right\rbrace
	\end{equation}
	where $\beta_* = \frac{\beta}{\beta -1}$, and $c_\beta(\rho) = (1+ \beta(\beta-1) \rho)^{\frac{1}{\beta}}$.

The following proposition shows that DRO can handle subpopulation shift under the domain-oblivious setting. The only information DRO needs during training is $\alpha$, the ratio between the size of the smallest domain and the size of the population. See the proof in Appendix \ref{proof:dro-fwod-generalized}.

\begin{prop}
	\label{prop:dro-fwod-generalized} 
	Let $\alpha = \min_{k=1,\cdots,K} P(\gD_k)$ be the minimal group size, and define $\rho = f_{\beta}(\frac{1}{\alpha})$. Then
	\begin{equation}
	\label{eqn:dro-fwod-generalized}
	\gR_{\max}(\theta;P) \leq  \gR_{D_{\beta},  \rho}(\theta; P)
	\end{equation}
\end{prop}

While the Cressie-Read formulation only defines the $f$-divergence for finite $\beta \in (1, +\infty)$, it can be shown that the dual characterization is valid for $\beta = \infty$ as well, for which the DORO risk becomes the well-known \textit{conditional value-at-risk} (CVaR) (See e.g. \cite{duchi2018learning}, Example 3).  In our theoretical analysis and experiments, we delve into two most widely-used sepecial cases of the Cressie-Read family: (i) $\beta=\infty$, which corresponds to CVaR; (ii) $\beta=2$, which corresponds to \textit{$\chi^2$-DRO risk} used in \cite{pmlr-v80-hashimoto18a}. Table \ref{tab:cressie-read} summarizes the relevant quatities in these two special cases.

\begin{table}[!ht]
\vskip -0.1in
\caption{CVaR and $\chi^2$-DRO. $\alpha$ is the ratio between the size of the smallest domain and the size of the population.}
\label{tab:cressie-read}
\vskip 0.15in
\begin{center}
\begin{small}
\begin{tabular}{c|cc}
\toprule
 & \textbf{CVaR} & \textbf{$\chi^2$-DRO}  \\
\midrule
$\beta$    & $\infty$ & 2 \\
$\beta_*$ & 1 & 2 \\ 
$\rho$ & $-\log (\alpha)$ & $\frac{1}{2}(\frac{1}{\alpha}-1)^2$ \\[3pt]
$c_\beta (\rho)$ & $\alpha^{-1}$ & $\sqrt{1+(\frac{1}{\alpha}-1)^2}$ \\
$D_\beta (Q \parallel P)$ & $\sup \log \frac{dQ}{dP}$ & $\frac{1}{2}\int (dQ /dP - 1)^2 dP$ \\[3pt]
DRO Risk & $\cvar_{\alpha}(\theta;P)$ & $\gR_{D_{\chi^2}, \rho}(\theta;P)$ \\
\bottomrule
\end{tabular}
\end{small}
\end{center}
\vskip -0.1in
\end{table}

For example, the dual form of CVaR is
\begin{equation}
\label{eqn:cvar-dual}
\cvar_{\alpha}(\theta;P) = \inf_{\eta \in \R} \{ \alpha^{-1} \E_{P}[(\ell(\theta;Z) - \eta)_+] + \eta \}
\end{equation}

It is easy to see that the optimal $\eta$ of (\ref{eqn:cvar-dual}) is the $\alpha$-quantile of $l(\theta;Z)$ defined as
\begin{equation}
\label{eqn:quantile}
q_\theta(\alpha) = \inf_{q}\{ P_{Z\sim P}(\ell(\theta;Z) > q) \leq \alpha \}
\end{equation}

The dual form (\ref{eqn:cvar-dual}) shows that CVaR in effect minimizes the expected risk on the worst $\alpha$ portion of the training data.


The following corollary of Proposition \ref{prop:dro-fwod-generalized} shows that both $\cvar_{\alpha}(\theta;P)$ and $\gR_{D_{\chi^2},  \rho}(\theta; P)$ are upper bounds of $\gR_{\max}(\theta;P)$, so that minimizing either of them guarantees a small worst-case risk (see the proof in Appendix \ref{proof:dro-fwod}):

\begin{cor}
\label{prop:dro-fwod} 
Let $\alpha = \min_{k=1,\cdots,K} P(\gD_k)$ be the minimal group size, and $\rho = \frac{1}{2}(\frac{1}{\alpha}-1)^2$. Then
\begin{equation}
\label{eqn:dro-fwod}
\gR_{\max}(\theta;P) \leq \cvar_{\alpha}(\theta;P) \leq \gR_{D_{\chi^2},  \rho}(\theta; P)
\end{equation}
\end{cor}

\section{DRO is Sensitive to Outliers}
\label{sec:dro-not-robust}

Although the construction of DRO aims to be effective against subpopulation shift as detailed in the previous section, when applied to real tasks DRO is found to have poor and unstable performance. After some examination, we pinpoint one direct cause of this phenomenon: the vulnerablity of DRO to outliers that widely exist in modern datasets. In this section, we will provide some intriguing experimental results to show that:
\begin{enumerate}
    \item DRO methods have poor and unstable performances.
    \item Sensitivity to outliers is a direct cause of DRO's poor performance. To support this argument, we show that DRO becomes good and stable on a ``clean'' dataset constructed by removing the outliers from the original dataset, and new outliers added to this ``clean'' dataset compromise DRO's performance and stability.
\end{enumerate}

We conduct experiments on COMPAS \cite{larson2016we}, a recidivism prediction dataset with 5049 training instances (after preprocessing and train-test splitting). We select two features as protected features: race and sex. The two protected features define four overlapping demographic groups: \textit{White}, \textit{Others}, \textit{Male} and \textit{Female}. A two-layer feed-forward neural network with ReLU activations  is used as the classification model. We train three models on this dataset with ERM, CVaR and $\chi^2$-DRO. Then we remove the outliers from the training set using the following procedure: We first train a model with ERM, and then remove 200 training instances that incur the highest loss on this model, as outliers are likely to have poorer fit. Then we reinitialize the model, train it on the new training set with ERM, and remove 200 more instances with the highest loss from the new training set. This process is repeated 5 times, so that 1000 training instances are removed and we get a new training set with 4049 instances. Note that this procedure is not guaranteed to remove all outliers and retain all inliers, but is sufficient for the purposes of our demonstration. We then run the three algorithms again on this \textit{same} ``clean'' training set.

\begin{figure}[!t]
     \centering
     \begin{subfigure}[b]{0.48\linewidth}
         \centering
         \includegraphics[width=\textwidth]{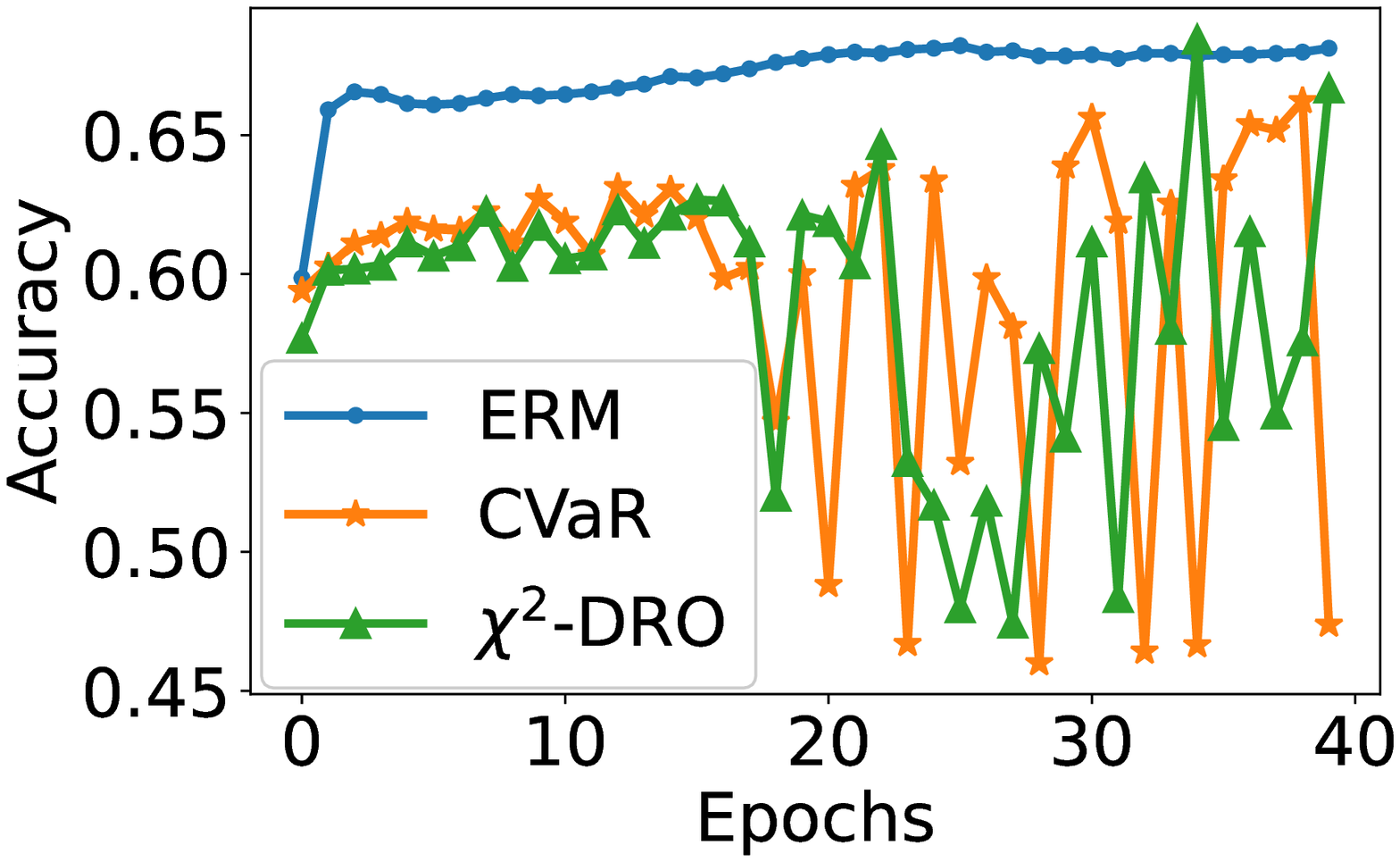}
         \caption{Average (Original)}
     \end{subfigure}
     \hfill
     \begin{subfigure}[b]{0.48\linewidth}
         \centering
         \includegraphics[width=\textwidth]{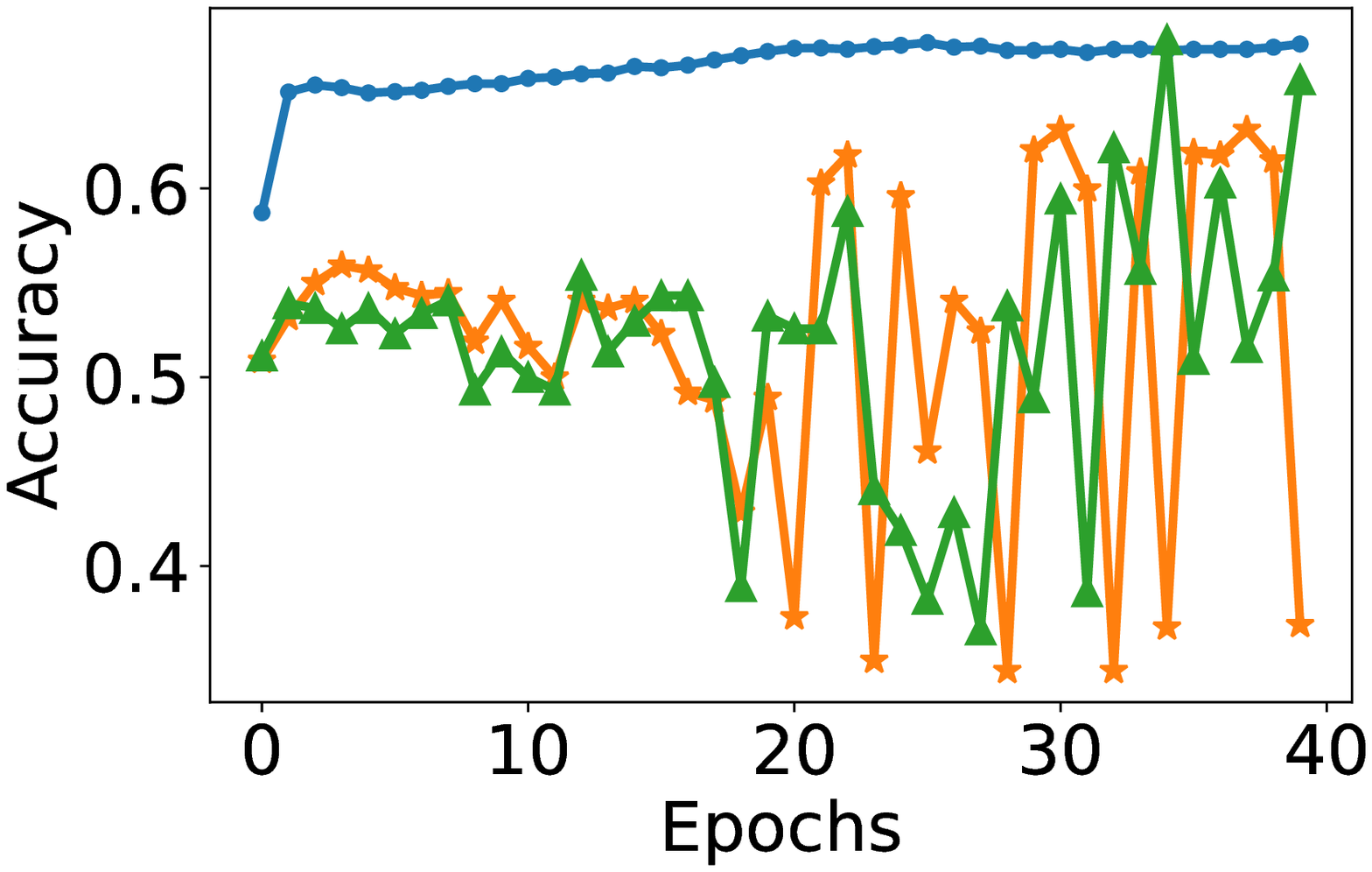}
         \caption{Worst (Original)}
     \end{subfigure}
      \begin{subfigure}[b]{0.48\linewidth}
     \centering
     \includegraphics[width=\textwidth]{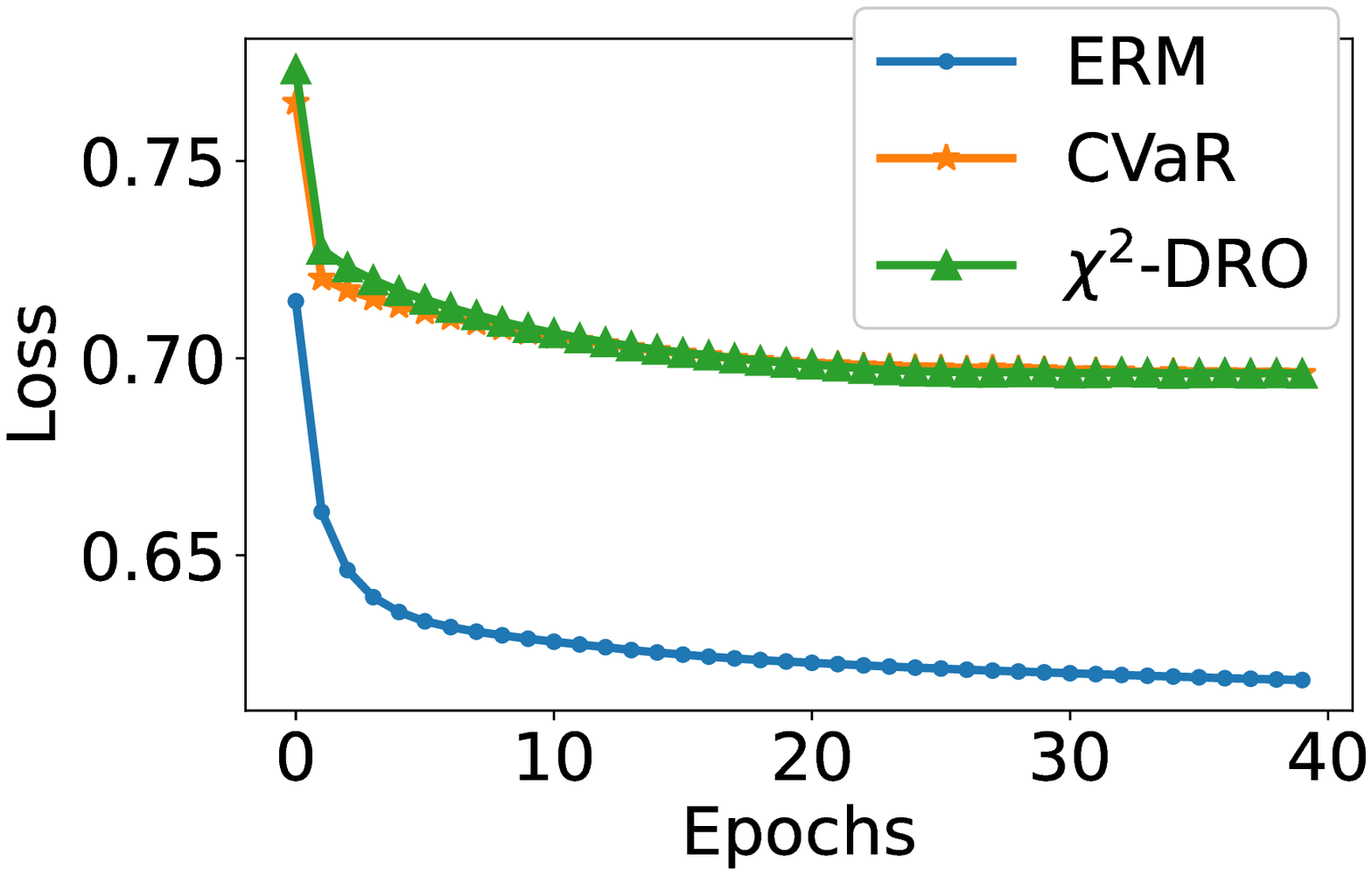}
    \caption{\camready{Train Loss (Original)}}
     \end{subfigure}
     \hfill
     \begin{subfigure}[b]{0.48\linewidth}
     \centering
     \includegraphics[width=\textwidth]{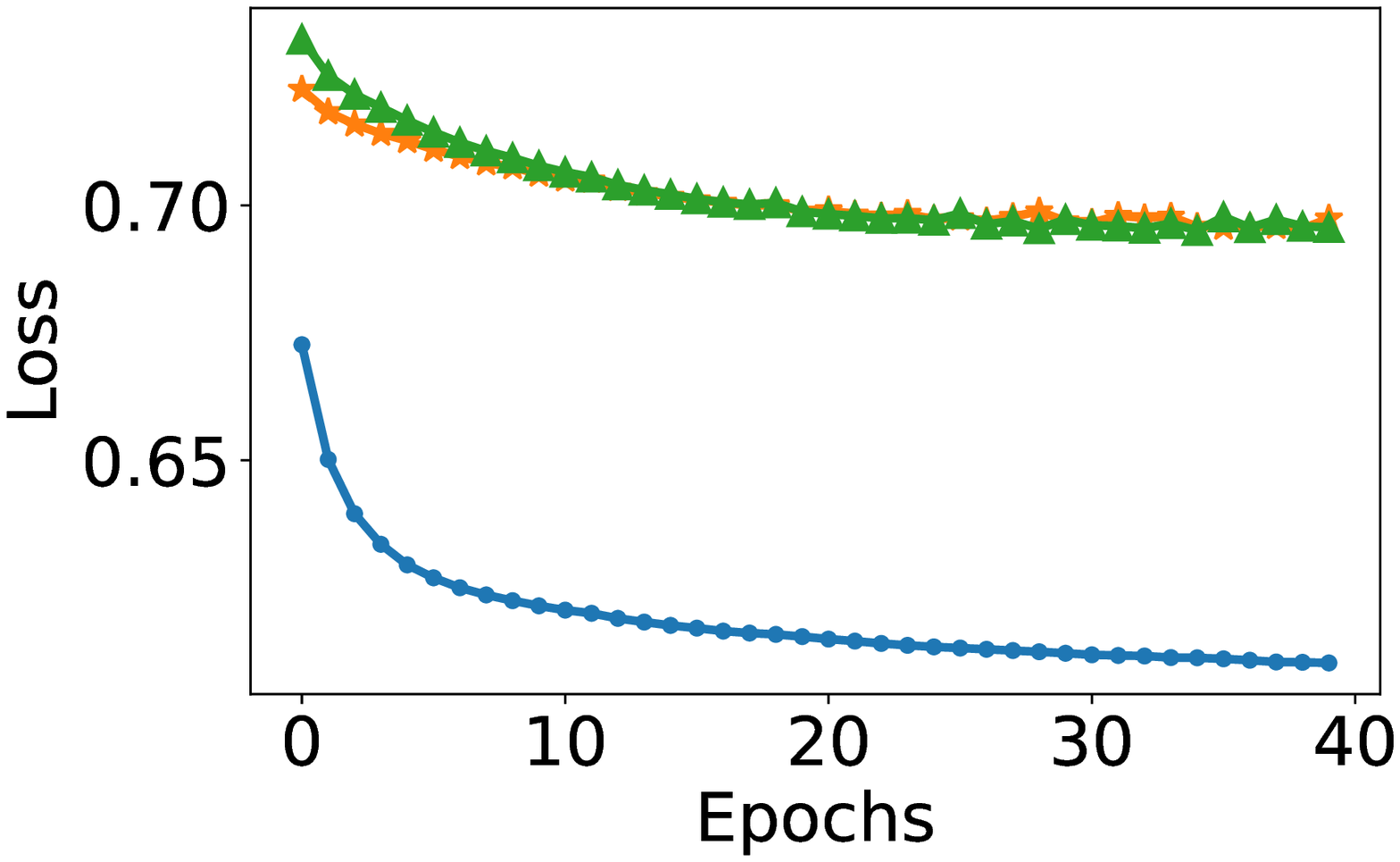}
     \caption{\camready{Test Loss (Original)}}
     \end{subfigure}
     \begin{subfigure}[b]{0.48\linewidth}
         \centering
         \includegraphics[width=\textwidth]{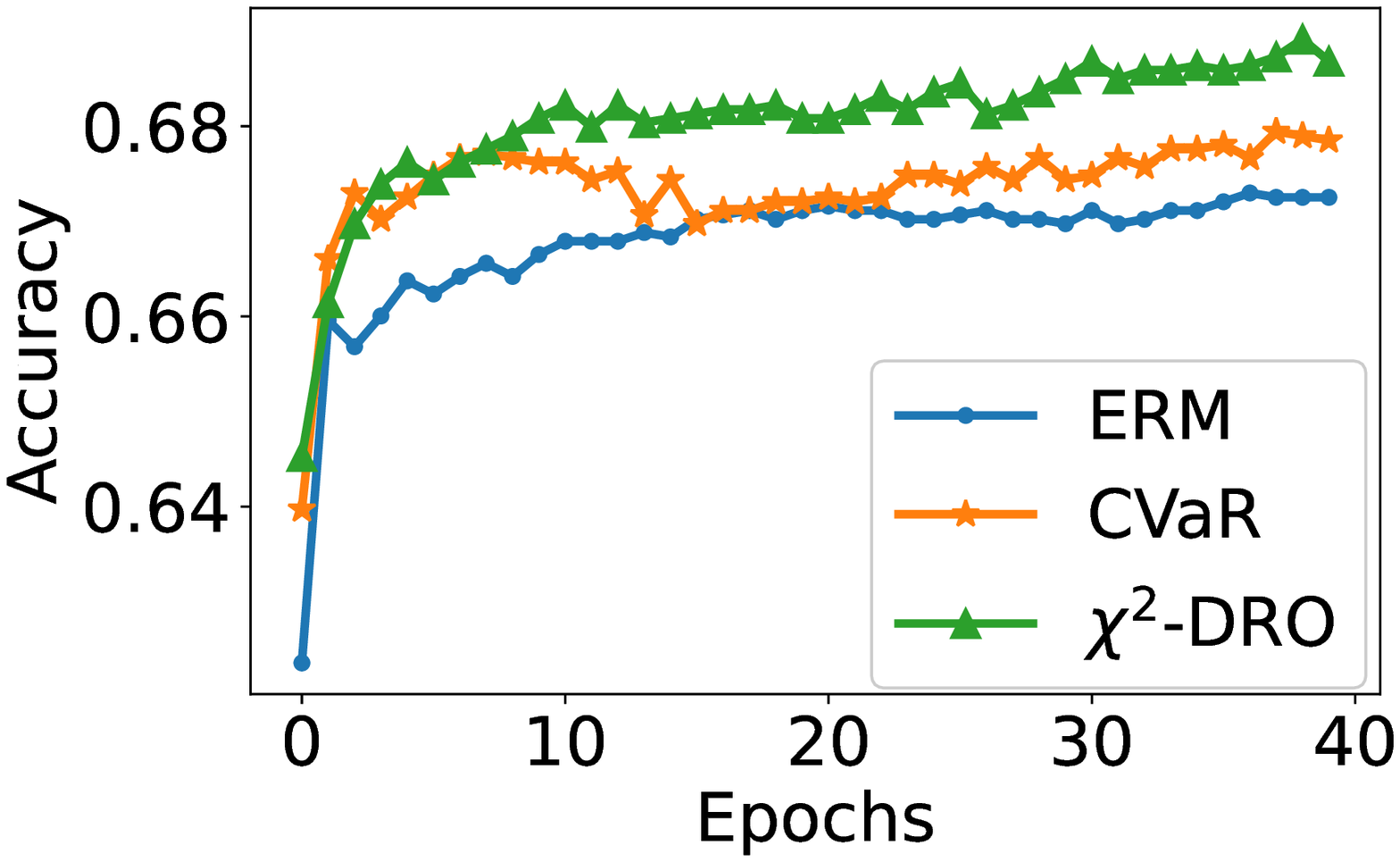}
        \caption{Average (Outliers removed)}
     \end{subfigure}
     \hfill
     \begin{subfigure}[b]{0.48\linewidth}
         \centering
         \includegraphics[width=\textwidth]{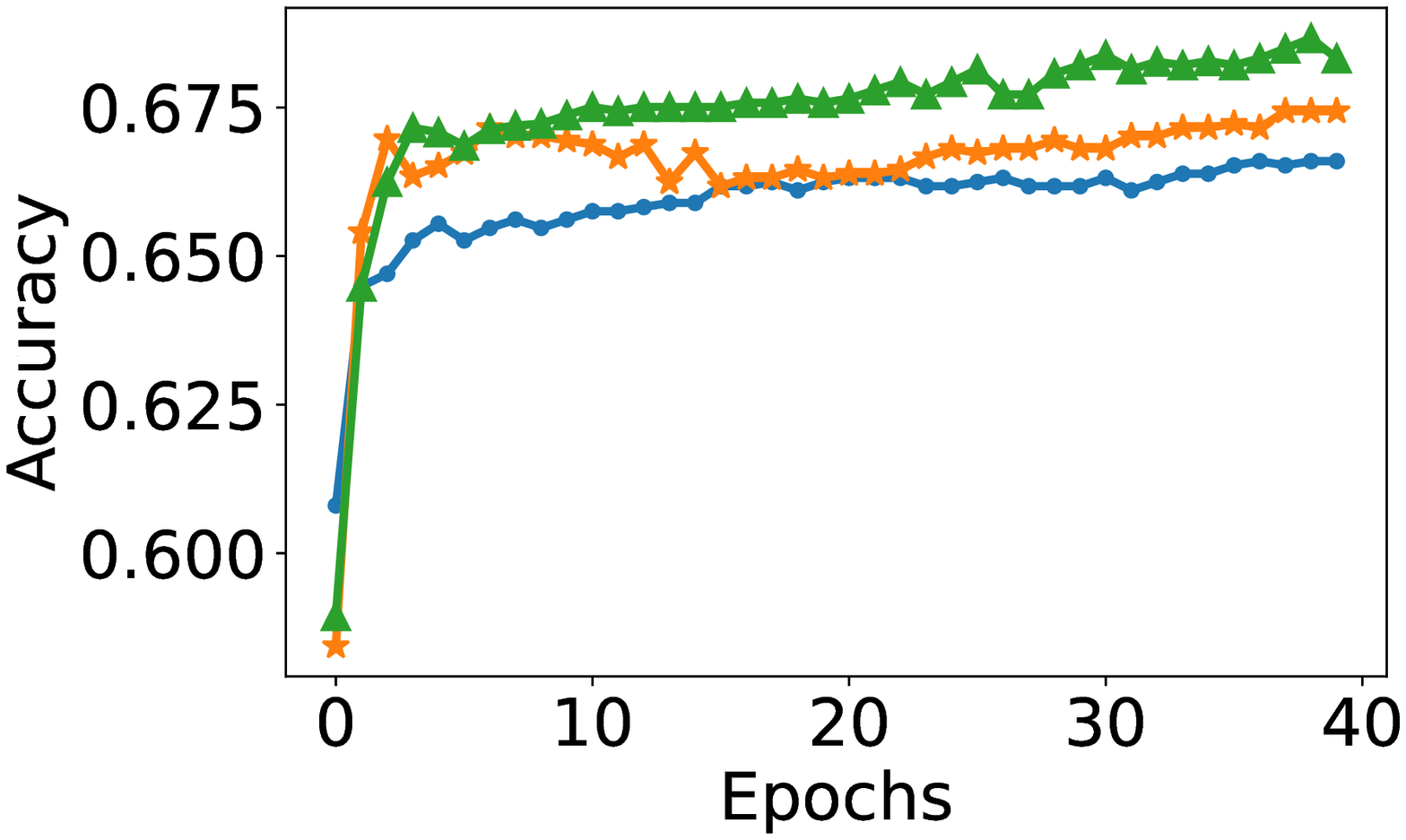}
         \caption{Worst (Outliers removed)}
     \end{subfigure}
     \begin{subfigure}[b]{0.48\linewidth}
         \centering
         \includegraphics[width=\textwidth]{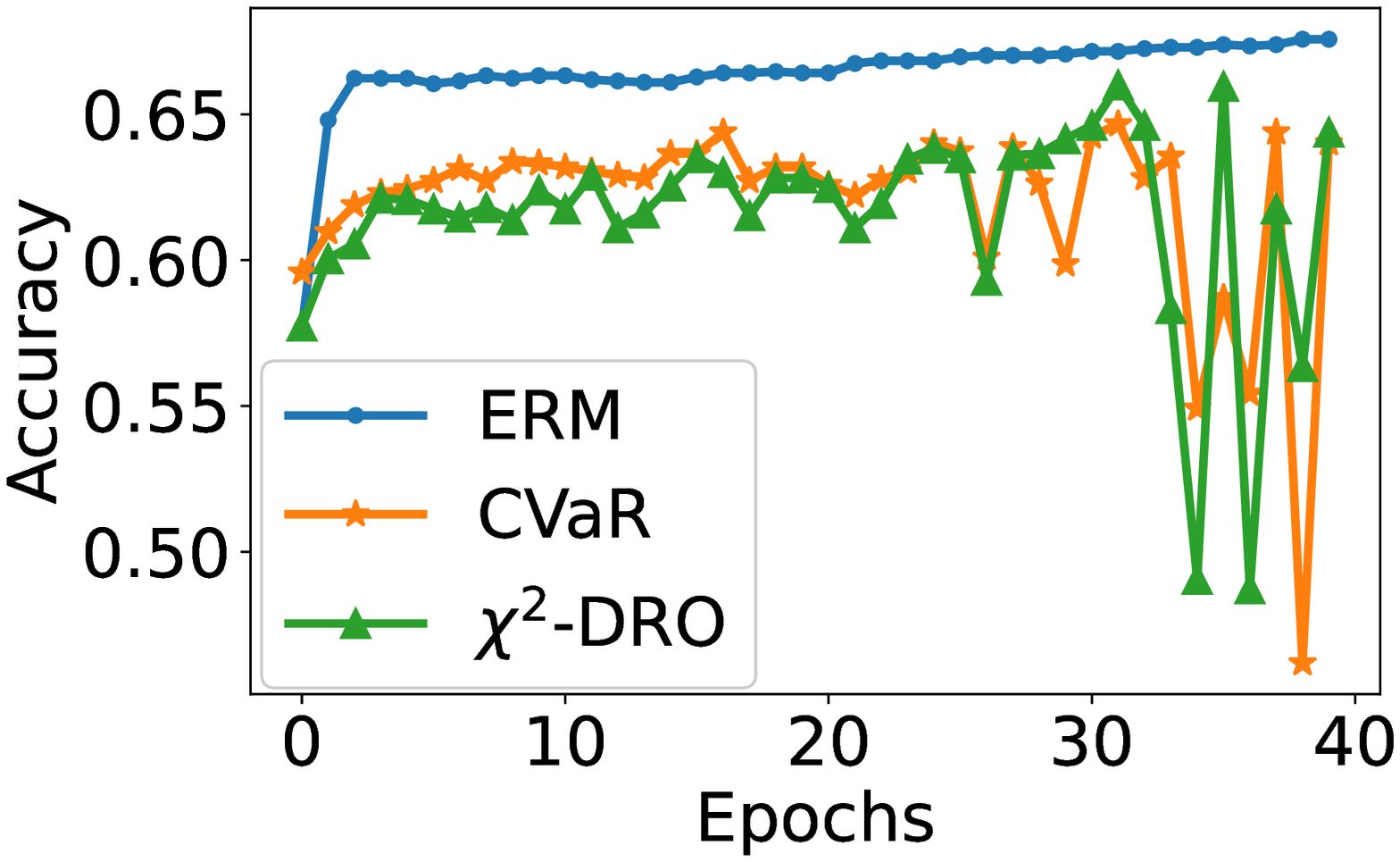}
        \caption{Average (Labels flipped)}
     \end{subfigure}
     \hfill
     \begin{subfigure}[b]{0.48\linewidth}
         \centering
         \includegraphics[width=\textwidth]{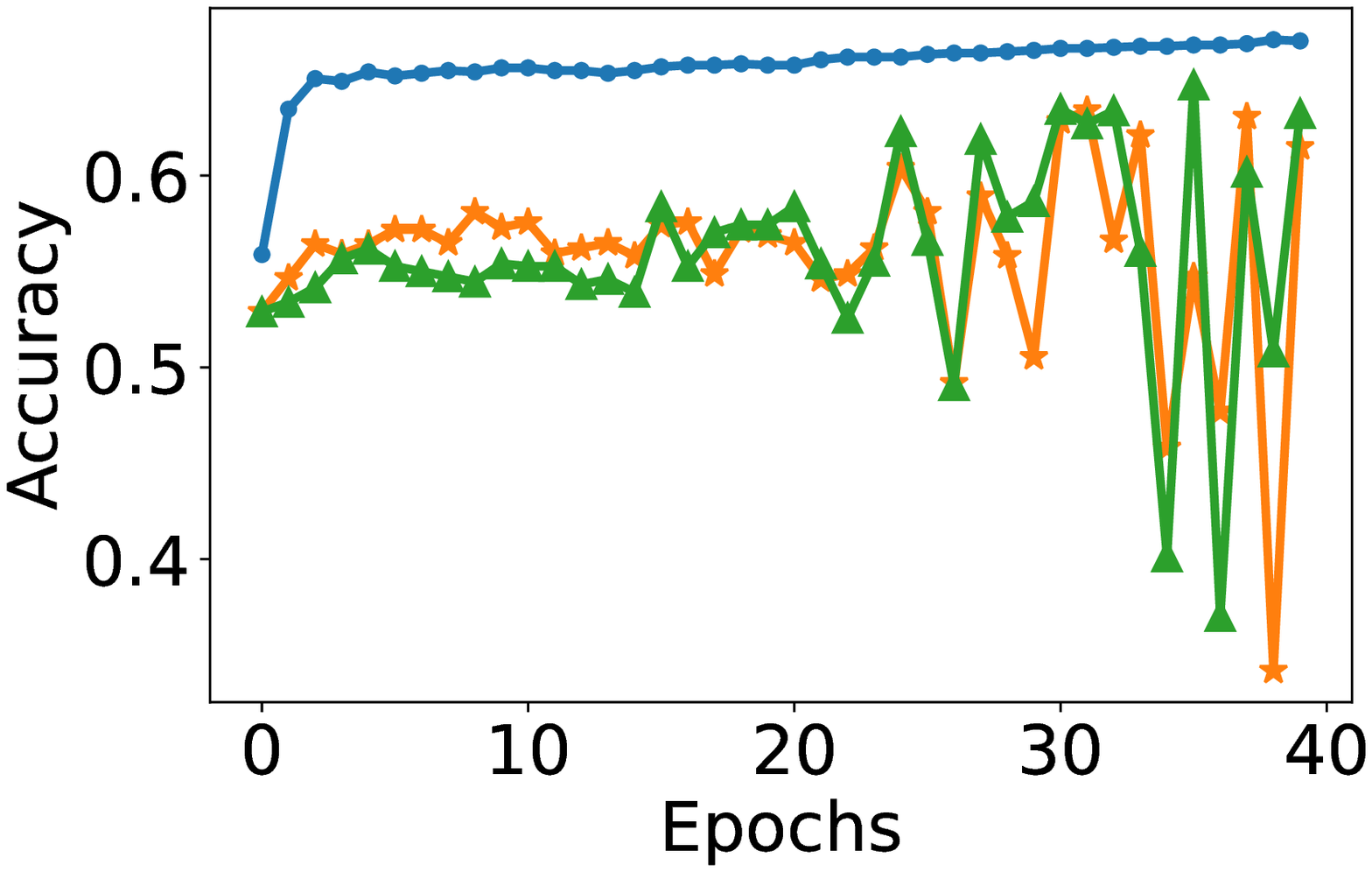}
         \caption{Worst (Labels flipped)}
     \end{subfigure}
     \begin{subfigure}[b]{0.48\linewidth}
         \centering
         \includegraphics[width=\textwidth]{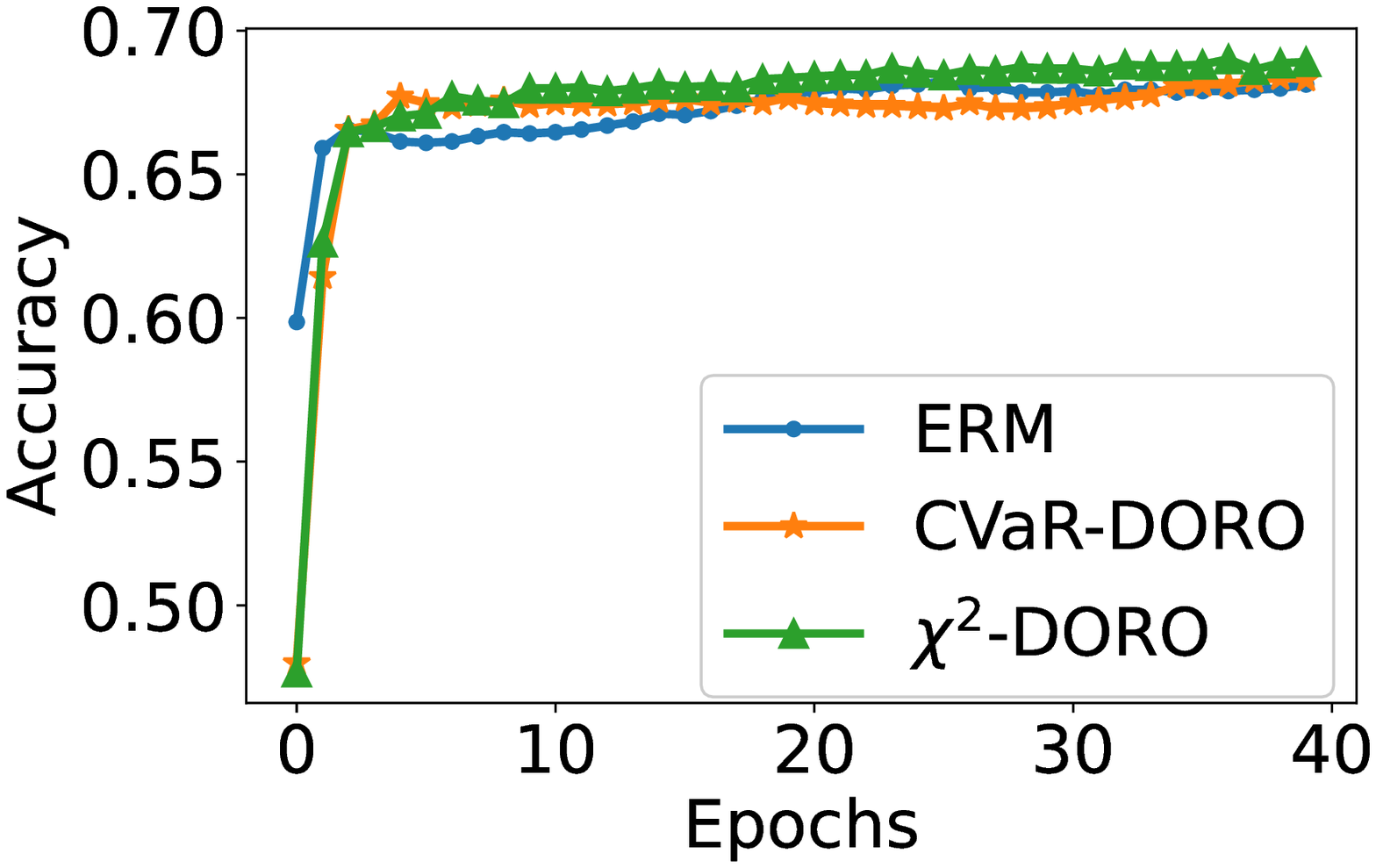}
        \caption{Average (Original)}
     \end{subfigure}
     \hfill
     \begin{subfigure}[b]{0.48\linewidth}
         \centering
         \includegraphics[width=\textwidth]{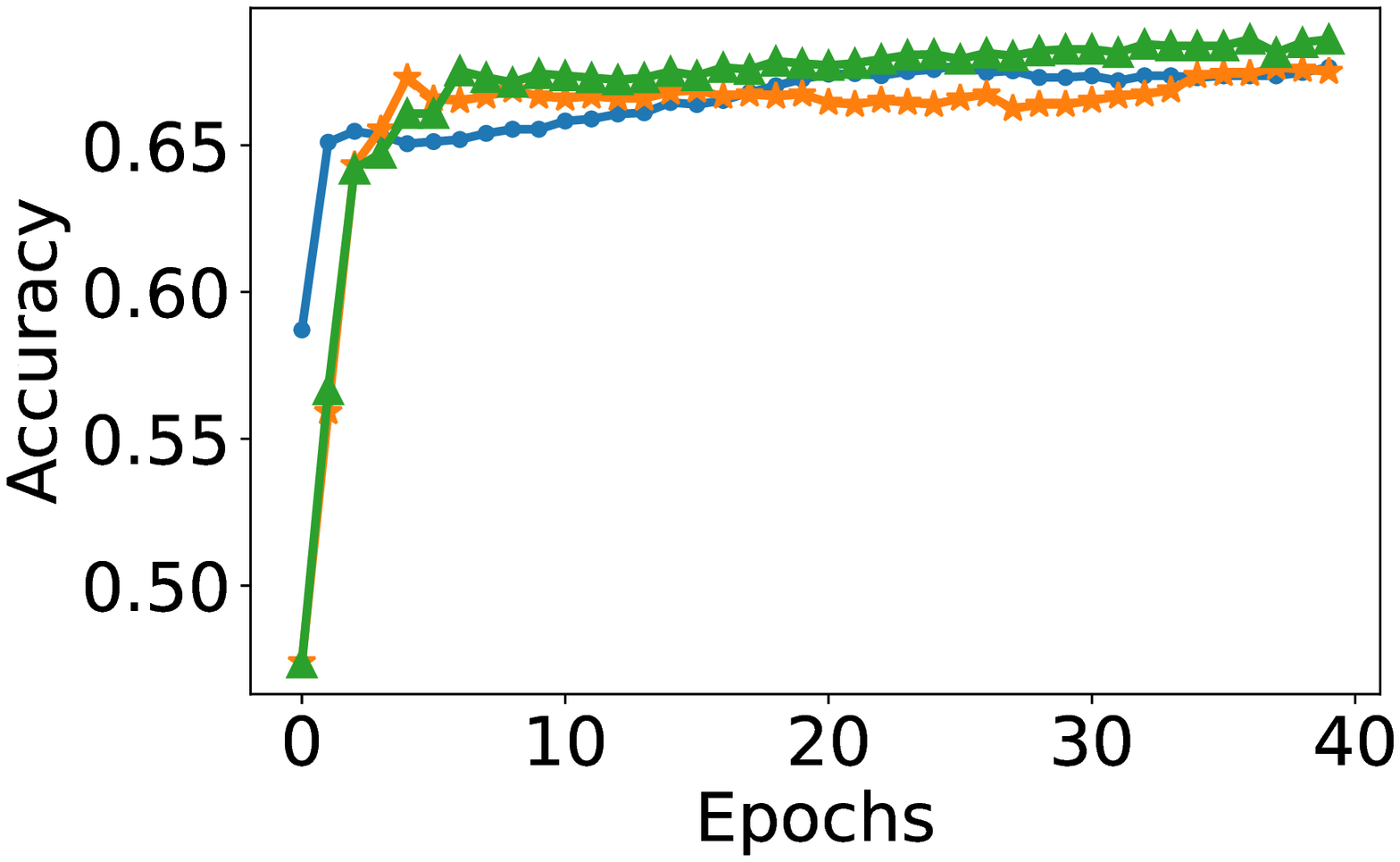}
         \caption{Worst (Original)}
     \end{subfigure}
    \caption{Average/Worst-case test accuracies on the COMPAS dataset (Original, ``clean'' with the outliers removed, and ``clean with label noise'' with 20\% of the labels flipped). \camready{The second row shows the train/test loss of ERM and DRO on the original dataset (average over all samples).} The last row shows the performance of DORO on the original dataset.}
    \label{fig:acc-compas}
    \vskip -0.1in
\end{figure}

We plot the test accuracies (average and worst across four demographic groups) of the models achieved by the three methods in Figure \ref{fig:acc-compas}. The first row shows the results on the original dataset, and the second row shows the results on the ``clean'' dataset with the outliers removed. We can see that in the first row, for both average and worst-case test accuracies, the DRO curves are below the ERM curves and jumping up and down, which implies that DRO has lower performance than ERM and is very unstable on the original dataset. However, the third row shows that DRO becomes good and stable after the outliers are removed. \camready{For comparison, in the second row we plot the train/test loss on the original dataset of the three methods (for ERM we plot the ERM loss, and for DRO we plot the corresponding DRO loss). The train and test losses of DRO descend steadily while the average and worst-case accuracies jump up and down, which indicates that the instability is not an optimization issue, but rather stems from the existence of outliers.} It should also be emphasized that these outliers naturally exist in the original dataset since no outliers have been manually added yet.

To further substantiate our conclusion, we consider another common source of outliers: incorrect labels. We randomly flip 20\% of the labels of the ``clean'' COMPAS dataset with the outliers removed, and run the three training methods again. The results are plotted in the fourth row of Figure \ref{fig:acc-compas}, which shows that while the label noise just slightly influences ERM, it significantly downgrades the performance and stability of the two DRO methods.

Likewise, \cite{hu2018does} also found in their experiments that DRO had even lower performance than ERM (see their Table 1). Essentially, DRO methods minimize the expected risk on the worst portion of the training data, which contains a higher density of outliers than the whole population. Training on these instances naturally result in the observed bad performance of DRO.

In the next section we will propose DORO as a solution to the problem revealed by the experiments in this section. We plot the performances of the two DORO algorithms we implement in the last row of Figure \ref{fig:acc-compas}, which compared to the first row shows that DORO improves the performance and stability of DRO on the original dataset.

\section{DORO}
\label{sec:method}

\paragraph{Problem Setting} 
The goal is to train a model on a dataset with outliers to achieve high tail performance on the clean underlying data distribution $P$. Denote the observed contaminated training distribution by $\ptrain$. We formulate $\ptrain$ with Huber's $\epsilon$-contamination model \cite{huber1992robust}, in which the training instances are i.i.d. sampled from
\begin{equation}
\label{eqn:huber-eps}
\ptrain = (1-\epsilon)P + \epsilon \tilde{P}
\end{equation}
where $\tilde{P}$ is an \textit{arbitrary} outlier distribution, and $0 < \epsilon < \frac{1}{2} $ is the noise level. The objective is to minimize $\gR_{\max}(\theta;P)$, the worst-case risk over the clean distribution $P$.

\paragraph{DORO Risk}
We propose to minimize the following \textit{expected $\epsilon$-DORO risk}:
\begin{equation}
\label{eqn:robust-dro-risk}
\begin{aligned}
&\gR_{D,\rho,\epsilon}(\theta;\ptrain) = \\
&\inf_{P'} \{ \gR_{D,\rho}(\theta;P'): \exists \tilde{P'} \textnormal{ s.t. } \ptrain = (1-\epsilon)P' + \epsilon \tilde{P'} \}
\end{aligned}
\end{equation}
The DORO risk is motivated by the following intuition: we would like the algorithm to avoid the ``hardest'' instances that are likely to be outliers, and the optimal $P'$ of (\ref{eqn:robust-dro-risk}) consists of the ``easiest'' $(1-\epsilon)$-portion of the training set given the current model parameters $\theta$. The $\epsilon$ in DORO is a hyperparameter selected by the user since the real noise level of the dataset is unknown. Let the real noise level of $\ptrain$ be $\epsilon_0$. For any $\epsilon \geq \epsilon_0$, there exist $\tilde{P}_0$ and $\tilde{P}$ such that $\ptrain =(1-\epsilon_0)P + \epsilon_0 \tilde{P}_0=(1-\epsilon)P+\epsilon \tilde{P}$, so we only need to make sure that $\epsilon$ is not less than the real noise level.

The following proposition provides the formula for computing the DORO risk for the Cressie-Read family (See the proof in Appendix \ref{proof:prop-robust-dro-cressie}):
\begin{prop}
\label{prop:robust-dro-cressie}
Let $\ell$ be a continuous non-negative loss function, and suppose $\ptrain$ is a continuous distribution. Then the formula for computing the DORO risk with $D_\beta$ is
\begin{equation}
\label{eqn:robust-dro-cressie}
\begin{aligned}
& \gR_{D_\beta, \rho, \epsilon}(\theta;\ptrain) =  \\
& \quad \inf_\eta \{ c_\beta(\rho) \E_{Z \sim \ptrain}[(\ell(\theta;Z)-\eta)_+^{\beta_*} \mid \\
& \quad P_{Z' \sim \ptrain} (\ell(\theta;Z')>\ell(\theta;Z))\geq\epsilon]^{\frac{1}{\beta_*} } + \eta \}
\end{aligned}
\end{equation}
\end{prop}

\camready{
\paragraph{Remark}
In Proposition \ref{prop:robust-dro-cressie}, we assume the continuity of $\ptrain$ to keep the formula simple. For an arbitrary distribution $\ptrain$, we can obtain a similar formula, but the formula is much more complex than (\ref{eqn:robust-dro-cressie}). The general formula can be found in Appendix \ref{proof:extend-formula}.
}

\begin{algorithm}[!t]
\caption{DORO with $D_\beta$ Divergence}
\label{alg:robust-dro-cressie}
\begin{algorithmic}
\STATE \textbf{Input:} Batch size $n$, outlier fraction $\epsilon$, minimal group size $\alpha$
\FOR{each iteration}
\STATE Sample a batch $z_1,\cdots,z_n \sim \ptrain$
\STATE Compute losses: $\ell_i = \ell(\theta,z_i)$ for $i=1,\cdots,n$
\STATE Sort the losses: $\ell_{i_1} \geq \cdots \geq \ell_{i_n}$
\STATE Find $\eta^* = \argmin_{\eta} F(\theta,\eta)$ where $F(\theta,\eta) = c_\beta(\rho) \cdot [ \frac{1}{n-\lfloor \epsilon n \rfloor} \sum_{j=\lfloor \epsilon n \rfloor+1}^{n} (\ell(\theta;z_{i_j})-\eta)_+^{\beta_*}]^{1/{\beta_*}} +\eta$
\STATE Update $\theta$ by one step to minimize $\ell(\theta)= F(\theta,\eta^*)$ with some gradient method
\ENDFOR
\end{algorithmic}
\end{algorithm}

With this formula, we develop Algorithm \ref{alg:robust-dro-cressie}. In the algorithm, we first order the batch samples according to their training losses, then find the optimal $\eta^*$ using some numerical method (we use Brent's method \cite{brent1971algorithm} in our implementation), and finally update $\theta$ with some gradient method. \camready{Note that generally it is difficult to find the minimizer of the DORO risk for neural networks, and our algorithm is inspired by the ITLM algorithm \cite{shen2019learning}, in which they proved that the optimization converges to ground truth for a few simple problems.} Particularly, using the quantities listed in Table \ref{tab:cressie-read}, we can implement CVaR-DORO and $\chi^2$-DORO. In the sections that follow, we will focus on the performances of CVaR-DORO and $\chi^2$-DORO in particular. We denote the CVaR-DORO risk by $\cvar_{\alpha,\epsilon}(\theta;\ptrain)$, and the $\chi^2$-DORO risk by $\gR_{D_{\chi^2},\rho,\epsilon}(\theta;\ptrain)$.

\section{Theoretical Analysis}
\label{sec:theory}

Having the DORO algorithms implemented, in this section we prove that DORO can effectively handle subpopulation shift in the presence of outliers. The proofs to the results in this section can be found in Appendix \ref{proof:sec-theory}. We summarize our theoretical results as follows:
\begin{enumerate}
    \item The minimizer of DORO over the contaminated distribution $\ptrain$ achieves a DRO risk close to the minimum over the clean distribution $P$ (Theorem \ref{thm:thm5}). We complement our analysis with information-theoretical lower bounds (Theorem \ref{thm:main_LB}) implying that the optimality gaps given by Theorem \ref{thm:thm5} are optimal.
    \item The worst-case risk $\gR_{\max}$ over $P$ is upper bounded by the DORO risk over $\ptrain$ times a constant factor (Theorem \ref{thm:effectiveness}). This result parallels Corollary \ref{prop:dro-fwod} in the uncontaminated setting and guarantees that minimizing the DORO risk over $\ptrain$ effectively minimizes $\gR_{\max}$ over $P$.
\end{enumerate}



Our results are based on the following lemma which lower bounds the DORO risk over $\ptrain$ by the infimum of the original DRO risk in a TV-ball centered at $P$:
\begin{lem}
\label{lem:robust-dro}
Let $\tv(P,Q)=\frac{1}{2}\int_{\gX \times \gY} |P(z) - Q(z)|dz$ be the total variation,  and $\ptrain$ be defined by (\ref{eqn:huber-eps}). Then the DORO risk can be lower bounded by:
\begin{equation}
\label{eqn:lem1}
\begin{aligned}
&\gR_{D,\rho,\epsilon}(\theta;\ptrain) \geq \\
& \quad \inf_{P''}\{ \gR_{D,\rho}(\theta;P''):  \tv (P, P'') \leq \frac{\epsilon}{1-\epsilon} \}
\end{aligned}
\end{equation}
\end{lem}

The main results we are about to present only require very mild assumptions. For the first result, we assume that $\ell$ has a bounded $(2k)$-th moment on $P$, a standard assumption in the robust statistics literature:

\begin{thm}	
	\label{thm:thm5}
	Let $\ptrain$ be defined by (\ref{eqn:huber-eps}). Denote the minimizer of the DORO risk by $\hat{\theta}$. If $\ell$ is non-negative, and $\ell(\hat{\theta};Z)$ has a bounded $(2k)$-th moment: $\E_{Z\sim P}[l(\hat{\theta}; Z)^{2k}] = \sigma_{2k}^{2k} < +\infty$, then we have:
	\begin{equation}
	\label{eqn:thm5}
	\begin{aligned}
	\cvar_\alpha(\hat{\theta};P) - \inf_\theta \cvar_\alpha (\theta;P) \leq 
	 O_{\alpha, k}(1) \sigma_{2k}  \epsilon^{1 - \frac{1}{2k}}
	\end{aligned}
	\end{equation}
	and if $k>1$, then we have:
	\begin{equation}
	\label{eqn:thm5-chi2}
	\begin{aligned}
	\gR_{D_{\chi^2},\rho}(\hat{\theta};P) - \inf_\theta 
	\gR_{D_{\chi^2},\rho}(\theta;P) \leq  O_{\rho, k}(1) \sigma_{2k}  \epsilon^{\left(\frac{1}{2} - \frac{1}{2k}\right)}
	\end{aligned}
	\end{equation}
\end{thm}

Furthermore, the above optimality gaps are optimal:
\begin{thm} \label{thm:main_LB}
	There exists a pair of $(P,\ptrain)$ where $\ptrain = (1-\epsilon) P + \epsilon P'$ and $P$ has uniformly bounded $2k$-th moment: $ \forall \theta \in \Theta$, $\E_{P}[l(\theta, Z)^{2k}] \leq \sigma_{2k}^{2k} $ such that for any learner with only access to $\ptrain$, the best achievable error in DRO over $P$ is lower bounded by 
	\begin{align}
	\label{eqn:thm7}
	  &\cvar_\alpha(\hat{\theta};P) - \inf_{\theta \in \Theta} \cvar_\alpha (\theta;P) \geq \Omega_{\alpha,k}(1) \sigma_{2k}  \epsilon^{1 - \frac{1}{2k}} \\ 
	  \label{eqn:thm7-chi2}
	  &\gR_{D_{\chi^2},\rho}(\hat{\theta};P) - \inf_{\theta \in \Theta}
	\gR_{D_{\chi^2},\rho}(\theta;P) \geq   \Omega_{\rho,k}(1) \sigma_{2k}  \epsilon^{\left(\frac{1}{2} - \frac{1}{2k}\right)}
	\end{align}
\end{thm}
We make a few remarks on these theoretical results. The $O(\epsilon^{1-\frac{1}{2k}})$ and $ O(\epsilon^{\frac{1}{2} - \frac{1}{2k}})$ rates resemble the existing works on robust mean/moment estimation, see e.g. \cite{DBLP:conf/stoc/KothariSS18,pmlr-v108-prasad20a}. The robust mean estimation problem can be seen as a special case of CVaR when $\alpha=1$, where CVaR of any $\theta$ is just the mean of $l(\theta, Z)$. On the other hand, the connection between CVaR and robust moment estimation can be built with the dual characterization (\ref{eqn:dro_dual}): for any \emph{fixed} dual variable $\eta$, evaluating the dual is nothing but a robust ($\beta_*$-th) moment estimation of the random variable $(l(\theta, Z)-\eta)_+$. However, the problem we are trying to tackle in the above theorems is more challenging, in the sense that (1) DRO risk involves taking infimum over all $\eta \in \R$, but the moments of $(l(\theta, Z)- \eta)_+$ are not \emph{uniformly} bounded for all possible $\eta$'s; and (2) the optimal dual variable $\eta^*$ can be very different even for distributions extremely close in total-variation distance. In Appendix \ref{proof:sec-theory} we discuss how to overcome these difficulties in detail.


Our second result is a robust analogue to Corollary \ref{prop:dro-fwod}: we show that the worst-case risk $\gR_{\max}$ can be upper bounded by a constant factor times the DORO risk $\cvar_{\alpha,\epsilon}$, under the very mild assumption that $\ell$ has a uniformly bounded second moment on $P$ and $\gR_{\max}$ is not exceedingly small:
\begin{thm}
\label{thm:effectiveness}
Let $\ptrain$ be defined by (\ref{eqn:huber-eps}). Let $\alpha=\min_{k=1,\cdots,K} P(\gD_k)$, and $\rho = \frac{1}{2}(\frac{1}{\alpha}-1)^2$. If $\ell(\theta;Z)$ is a non-negative loss function with a uniformly bounded second moment: $\E_{Z\sim P}[\ell(\theta;Z)^2] \leq \sigma^2$ for all $\theta$, then we have:
\begin{equation}
\label{eqn:thm-effectiveness}
\begin{aligned} 
\gR_{\max}(\theta;P) &\leq \max \{ 3 \cvar_{\alpha,\epsilon}(\theta;\ptrain), 3\alpha^{-1} \sigma \sqrt{\frac{\epsilon}{1-\epsilon}} \} \\ 
& \leq \max \{ 3 D_{\chi^2,\rho,\epsilon}(\theta;\ptrain), 3\alpha^{-1} \sigma \sqrt{\frac{\epsilon}{1-\epsilon}} \}
\end{aligned}
\end{equation}
\end{thm}
Note that a similar result can be derived under the bounded $2k$-th moment condition with different constants.
\section{Experiments}
\label{sec:experiments}
In this section, we conduct large-scale experiments on modern datasets. Our results show that DORO improves the performance and stability of DRO. We also analyze the effect of hyperparameters on DRO and DORO.

\subsection{Setup}

\paragraph{Datasets} 
Our goal is to apply DRO to real tasks with subpopulation shift on modern datasets. While many previous work used small tabular datasets such as COMPAS, these datasets are insufficient for our purpose. Therefore, apart from COMPAS, we use two large datasets: CelebA \cite{liu2015deep} and CivilComments-Wilds \cite{borkan2019nuanced,koh2020wilds}. CelebA is a widely used vision dataset with 162,770 training instances, and CivilComments-Wilds is a recently released language dataset with 269,038 training instances. Both datasets are captured in the wild and labeled by potentially biased humans, so they can reveal many challenges we need to face in practice.

We summarize the datasets we use as follows: (i) COMPAS: recidivism prediction, where the target is whether the person will reoffend in two years; (ii) CelebA: human face recognition, where the target is whether the person has blond hair; (iii) CivilComments-Wilds: toxicity identification, where the target is whether the user comment contains toxic contents. All targets are binary. For COMPAS, we randomly sample 70\% of the instances to be the training data (with a fixed random seed) and the rest is the validation/testing data. Both CelebA and CivilComments-Wilds have official train-validation-test splits, so we use them directly. 

\paragraph{Domain Definition}
On COMPAS we define 4 domains (subpopulations), and on CelebA and CivilComments-Wilds we define 16 domains for each. Our domain definitions cover several types of subpopulation shift, such as different demographic groups, class imbalance, labeling biases, confounding variables, etc. See Appendix \ref{app:domain-def} for details.


\paragraph{Training}
We use a two-layer feed-forward neural network activated by ReLU on COMPAS, a ResNet18 \cite{he2016deep} on CelebA, and a BERT-base-uncased model \cite{devlin-etal-2019-bert} on CivilComments-Wilds. On each dataset, we run ERM, CVaR, $\chi^2$-DRO, CVaR-DORO and $\chi^2$-DORO. Each algorithm is run 300 epochs on COMPAS, 30 epochs on CelebA and 5 epochs on CivilComments-Wilds. For each method we collect the model achieved at the end of every epoch, and select the best model through validation. (On CivilComments-Wilds we collect 5 models each epoch, one for every $\sim$20\% of the training instances.)

\paragraph{Model Selection}
To select the best model, we assume that the domain membership of each instance is available in the validation set, and select the model with the highest worst-case validation accuracy. This is an oracle strategy since it requires a domain-aware validation set. Over the course of our experiments, we have realized that model selection with no group labels during validation is a very hard problem. On the other hand, model selection has a huge impact on the performance of the final model. We include some preliminary discussions on this issue in Appendix \ref{app:model-select}. Since model selection is not the main focus of this paper, we pose it as an open question.

\subsection{Results}
\label{sec:exp-results}

\begin{table*}[!t]
\caption{The average and worst-case test accuracies of the best models achieved by different methods. (\%)}
\label{tab:results-acc}
\begin{center}
\vskip 0.06in
\begin{small}
\begin{tabular}{ cccc } 
 \toprule
 \textbf{Dataset} & \textbf{Method} & \textbf{Average Accuracy}  & \textbf{Worst-case Accuracy}\\ 
 \midrule
 \multirow{5}{*}{COMPAS} & ERM & $69.31 \pm 0.19$  & $68.83 \pm 0.18$ \\ 
 & CVaR & $68.52 \pm 0.31$ & $68.22 \pm 0.30$ \\ 
 & CVaR-DORO & $69.38 \pm 0.10$ & $69.11 \pm 0.05$ \\ 
 & $\chi^2$-DRO & $67.93 \pm 0.40$ & $67.32 \pm 0.60$  \\ 
 & $\chi^2$-DORO & $69.62 \pm 0.16$  & $69.22 \pm 0.11$ \\ 
 \midrule
 \multirow{5}{*}{CelebA} & ERM & $95.01 \pm 0.38$ & $53.94 \pm 2.02$ \\ 
 & CVaR & $82.83 \pm 1.33$ & $66.44 \pm 2.34$ \\ 
 & CVaR-DORO & $92.91 \pm 0.48$ & $72.17 \pm 3.14$  \\ 
 & $\chi^2$-DRO & $83.85 \pm 1.42$ & $67.76 \pm 3.22$ \\ 
 & $\chi^2$-DORO & $82.18 \pm 1.17$ & $68.33 \pm 1.79$ \\ 
 \midrule
 \multirow{5}{*}{CivilComments-Wilds} & ERM & $92.04 \pm 0.24$ & $64.62 \pm 2.48$ \\ 
 & CVaR & $89.11 \pm 0.76$ & $63.90 \pm 4.42$ \\ 
 & CVaR-DORO & $90.45 \pm 0.70$ & $68.00 \pm 2.10$ \\ 
 & $\chi^2$-DRO & $90.08 \pm 0.92$ & $65.55 \pm 1.51$ \\ 
 & $\chi^2$-DORO & $90.11 \pm 1.09$ & $67.19 \pm 2.51$ \\ 
 \bottomrule
\end{tabular}
\end{small}
\end{center}
\vskip -0.15in
\end{table*}

\begin{table}[!t]
\vskip -0.1in
\caption{Standard deviations of average/worst-case test accuracies during training on CelebA. ($\alpha=0.1$ for CVaR/CVaR-DORO; $\alpha=0.3$ for $\chi^2$-DRO/$\chi^2$-DORO. $\epsilon=0.01$) (\%)}
\label{tab:acc-std-celeba}
\vskip 0.06in
\begin{center}
\begin{small}
\begin{tabular}{c|cc}
\toprule
\textbf{Method} & \textbf{Average} & \textbf{Worst-case}   \\
\midrule
ERM & $0.73 \pm 0.06$ & $8.59 \pm 0.90$  \\ 
CVaR & $11.53 \pm 1.72$ & $21.47 \pm 0.71$ \\ 
CVaR-DORO & $4.03 \pm 1.57$ & $16.84 \pm 0.91$ \\ 
$\chi^2$-DRO & $8.88 \pm 2.98$ & $19.06 \pm 1.18$ \\ 
$\chi^2$-DORO & $1.60 \pm 0.34$ & $13.01 \pm 1.40$ \\ 
\bottomrule
\end{tabular}
\end{small}
\end{center}
\vskip -0.2in
\end{table}

\begin{figure}[!t]
     \centering
     \begin{subfigure}[b]{0.48\linewidth}
         \centering
         \includegraphics[width=\textwidth]{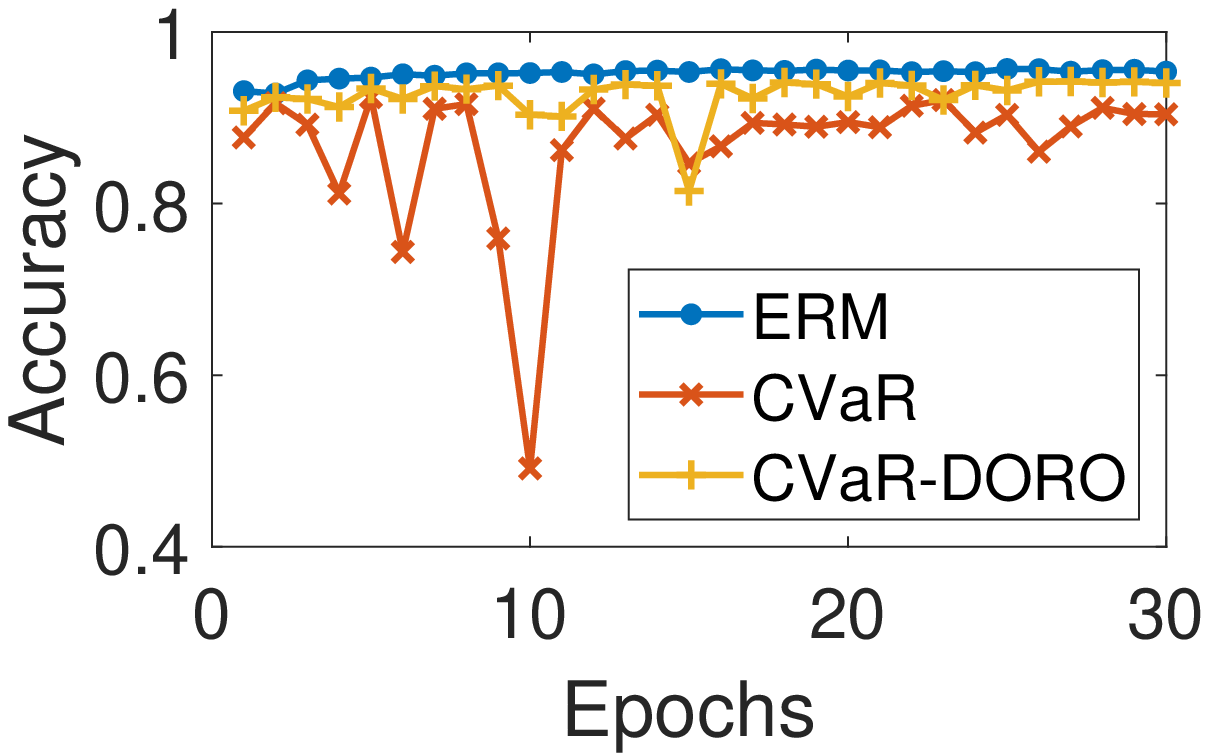}
         \caption{Average Accuracy}
         \label{fig:epochs-celeba-cvar-a}
     \end{subfigure}
     \hfill
     \begin{subfigure}[b]{0.48\linewidth}
         \centering
         \includegraphics[width=\textwidth]{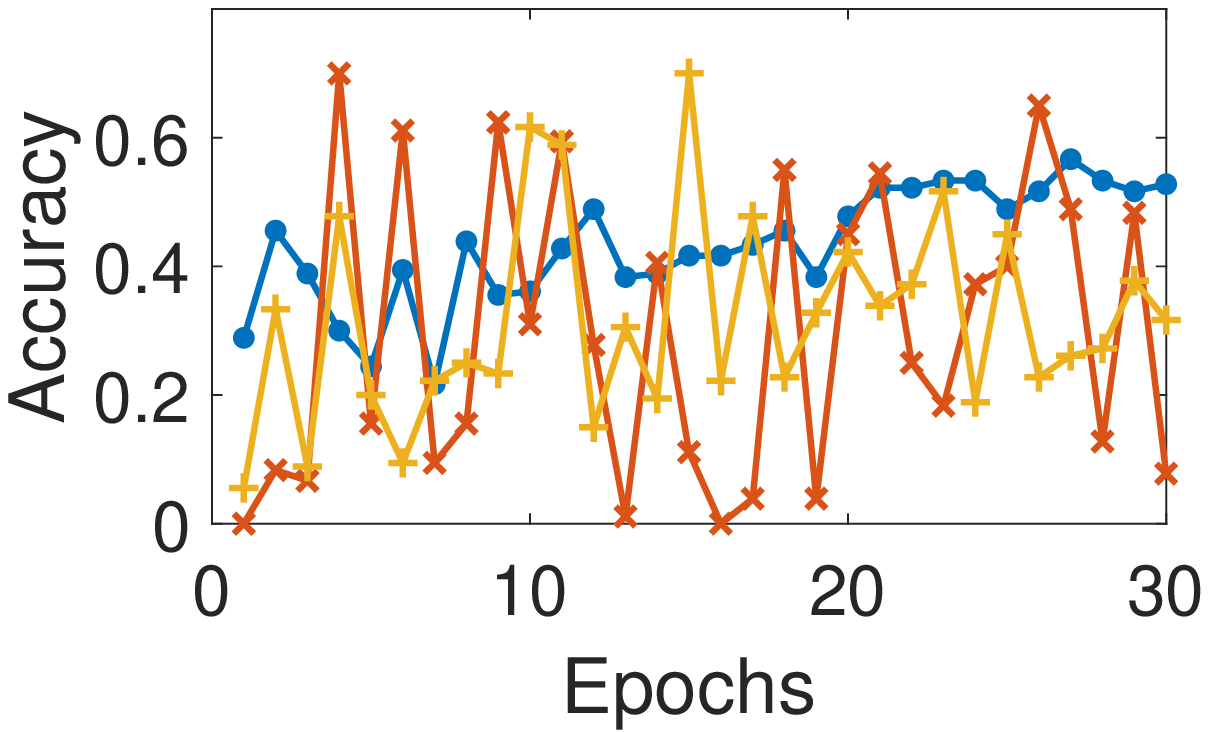}
         \caption{Worst-case Accuracy}
     \end{subfigure}
     \vskip -0.1in
    \caption{Test accuracies of CVaR and CVaR-DORO on CelebA ($\alpha=0.1$, $\epsilon = 0.01$).}
    \label{fig:epochs-celeba-cvar}
    \vskip -0.1in
\end{figure}

\begin{figure}[!t]
     \centering
     \begin{subfigure}[b]{0.48\linewidth}
         \centering
         \includegraphics[width=\textwidth]{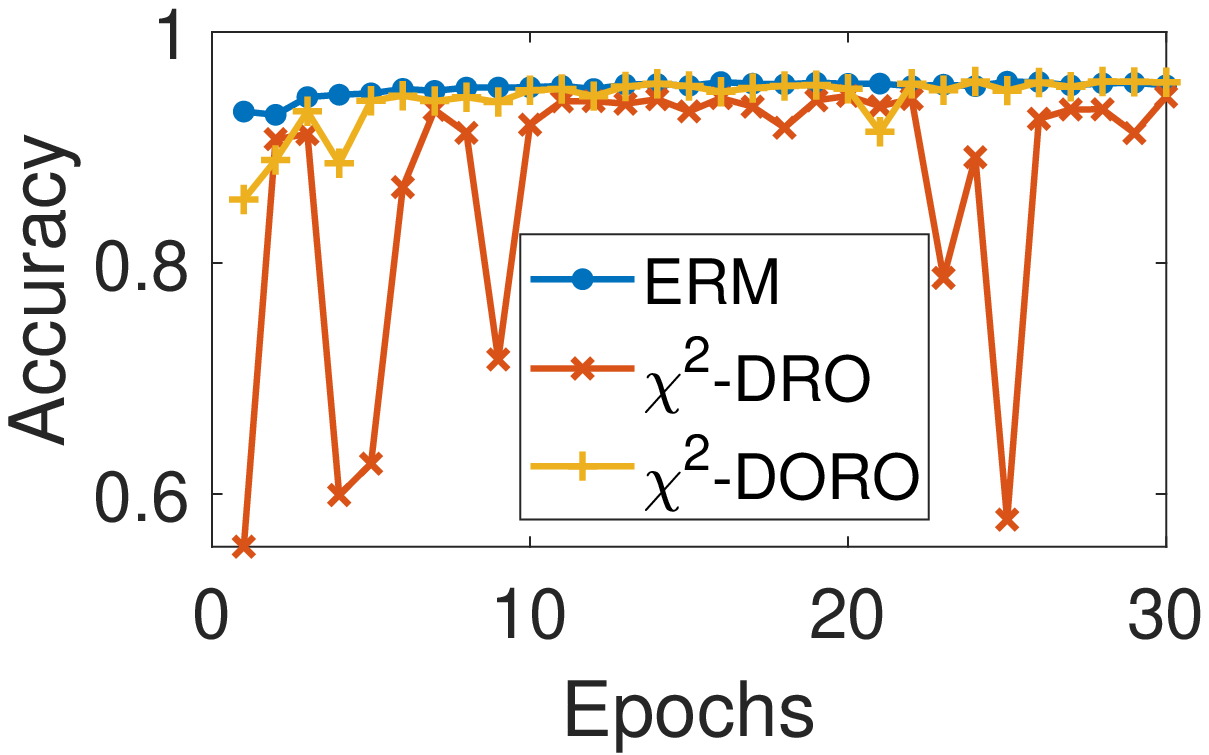}
         \caption{Average Accuracy}
         \label{fig:epochs-celeba-chisq-a}
     \end{subfigure}
     \hfill
     \begin{subfigure}[b]{0.48\linewidth}
         \centering
         \includegraphics[width=\textwidth]{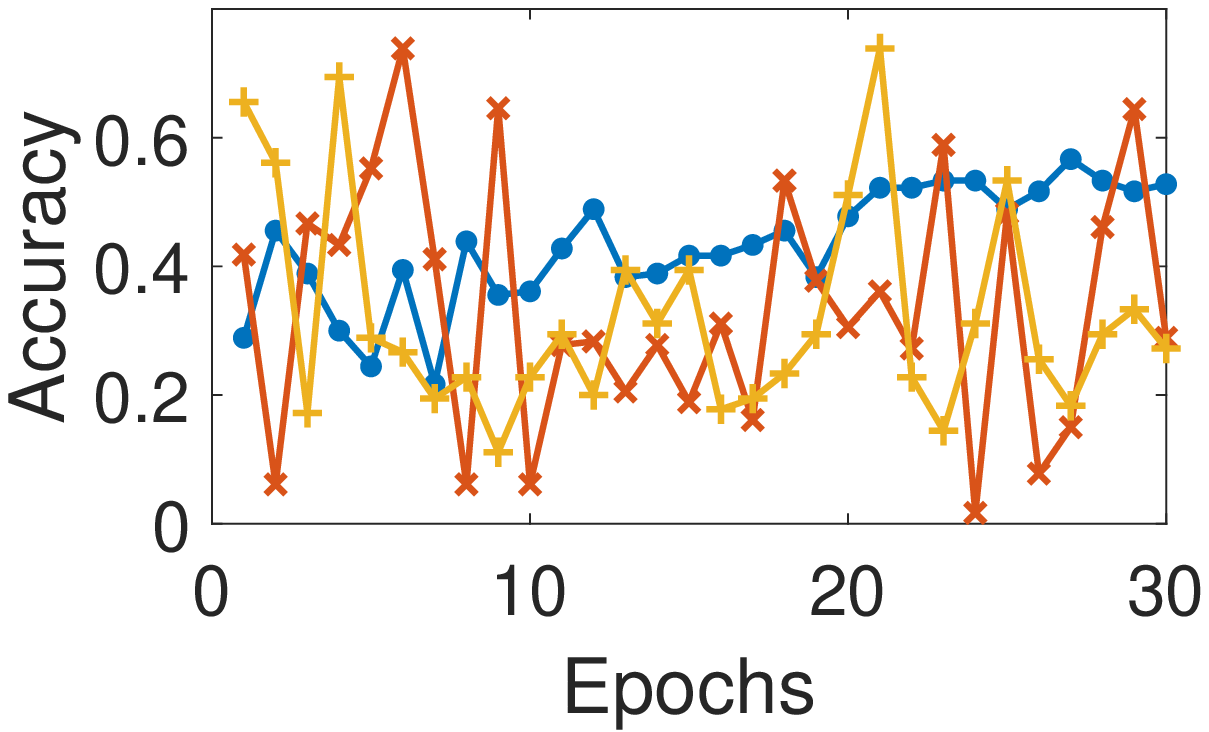}
         \caption{Worst-case Accuracy}
     \end{subfigure}
     \vskip -0.1in
    \caption{Test accuracies of $\chi^2$-DRO and $\chi^2$-DORO on CelebA ($\alpha=0.3$, $\epsilon = 0.01$).}
    \label{fig:epochs-celeba-chisq}
    \vskip -0.2in
\end{figure}

The 95\% confidence intervals of the mean test accuracies on each dataset are reported in Table \ref{tab:results-acc}. For every DRO and DORO method, we do a grid search to pick the best $\alpha$ and $\epsilon$ that achieve the best worst-case accuracy (see the optimal hyperparameters in Appendix \ref{app:best-param}). Each experiment is repeated 10 times on COMPAS and CelebA, and 5 times on CivilComments-Wilds with different random seeds. Table \ref{tab:results-acc} clearly shows that on all datasets, DORO consistently improves the average and worst-case accuracies of DRO. 

Next, we analyze the stability of the algorithms on the CelebA dataset. We use the $\alpha$ that achieves the optimal DRO performance for each of CVaR and $\chi^2$-DRO, and compare them to DORO with the same value of $\alpha$ and $\epsilon=0.01$. $\chi^2$-DRO achieves its optimal performance with a bigger $\alpha$ than CVaR because it is less stable. To quantitatively compare the stability, we compute the standard deviations of the test accuracies across epochs and report the results in Table \ref{tab:acc-std-celeba}. To further visualize the training dynamics, we run all algorithms with one fixed random seed, and plot the test accuracies during training in Figures \ref{fig:epochs-celeba-cvar} and \ref{fig:epochs-celeba-chisq}.
Table \ref{tab:acc-std-celeba} shows that the standard deviation of the test accuracy of DORO is smaller and in Figures \ref{fig:epochs-celeba-cvar-a} and \ref{fig:epochs-celeba-chisq-a} the DORO curves are flatter than the DRO curves, which implies that DORO improves the stability of DRO. Although it is hard to tell whether DORO has a more stable worst-case accuracy from the figures, our quantitative results in Table \ref{tab:acc-std-celeba} confirm that DORO has more stable worst-case test accuracies.

\subsection{Effect of Hyperparameters}
\label{sec:exp-effect}
\begin{figure}[!t]
     \centering
     \begin{subfigure}[b]{0.48\linewidth}
         \centering
         \includegraphics[width=\textwidth]{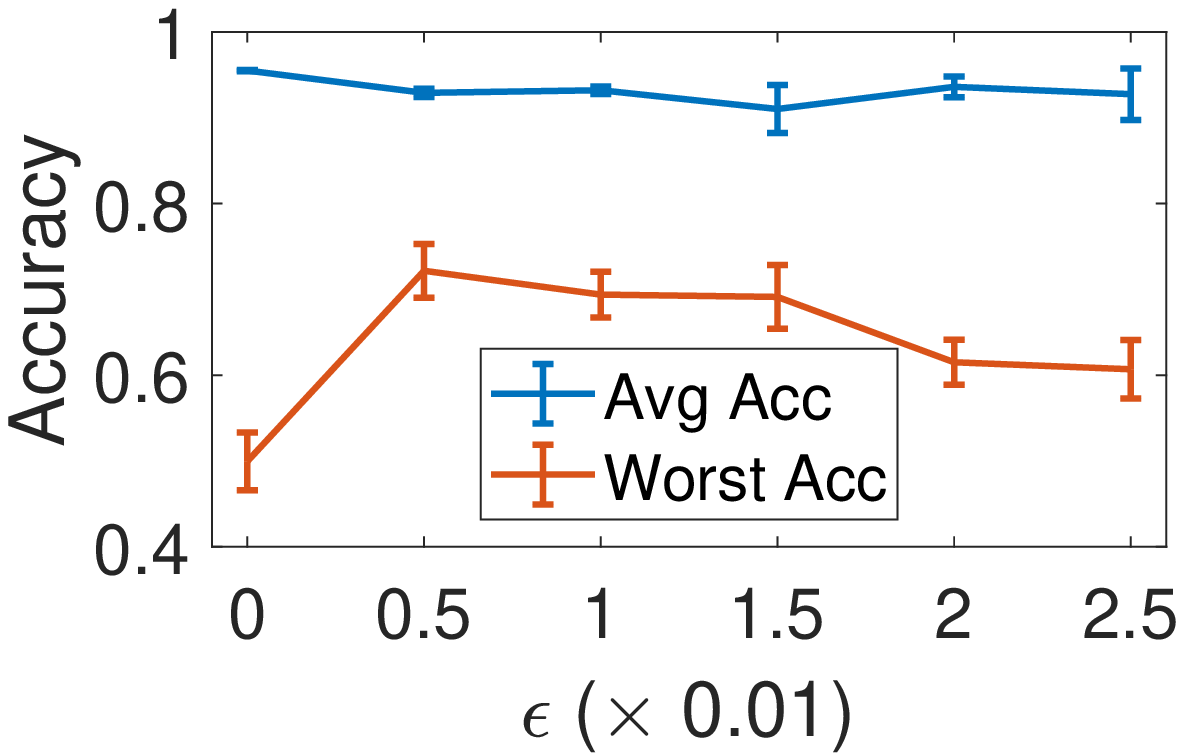}
         \caption{CVaR-DORO}
     \end{subfigure}
     \hfill
     \begin{subfigure}[b]{0.48\linewidth}
         \centering
         \includegraphics[width=\textwidth]{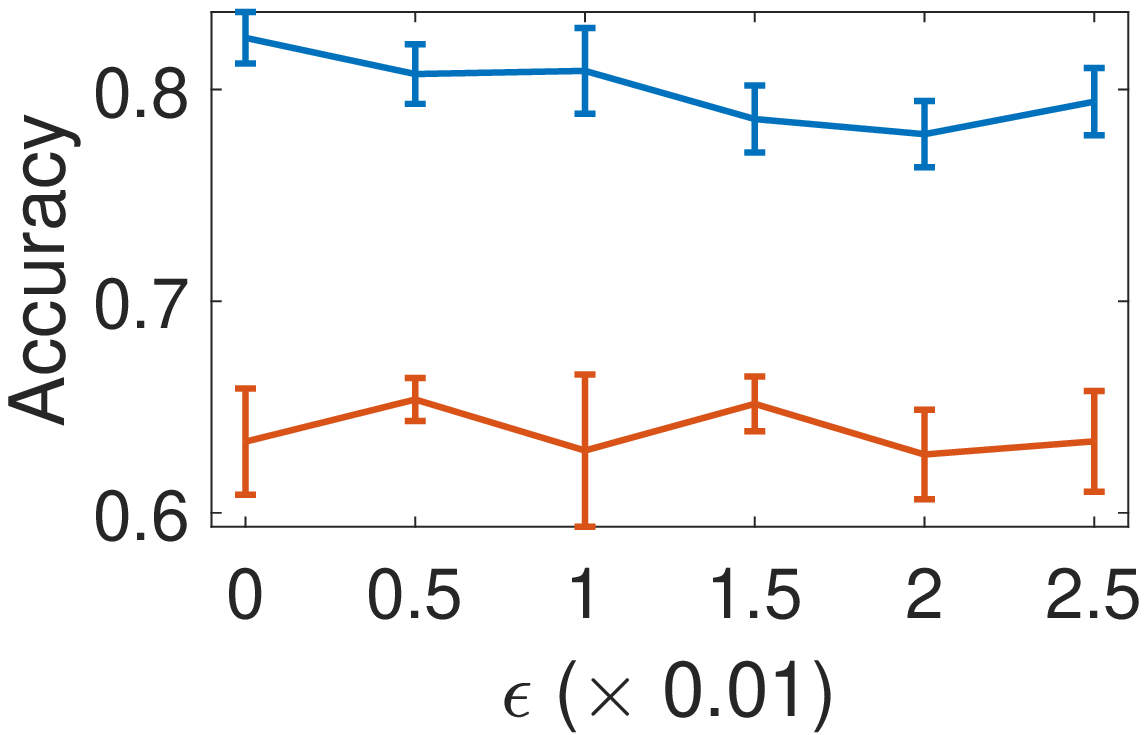}
         \caption{$\chi^2$-DORO}
     \end{subfigure}
    \caption{Effect of $\epsilon$ on the test accuracies of CVaR/$\chi^2$-DORO on CelebA ($\alpha=0.2$). DORO with $\epsilon=0$ is equivalent to DRO.}
    \label{fig:eps-celeba}
    \vskip -0.1in
\end{figure}

\begin{figure}[!t]
     \centering
     \begin{subfigure}[b]{0.48\linewidth}
         \centering
         \includegraphics[width=\textwidth]{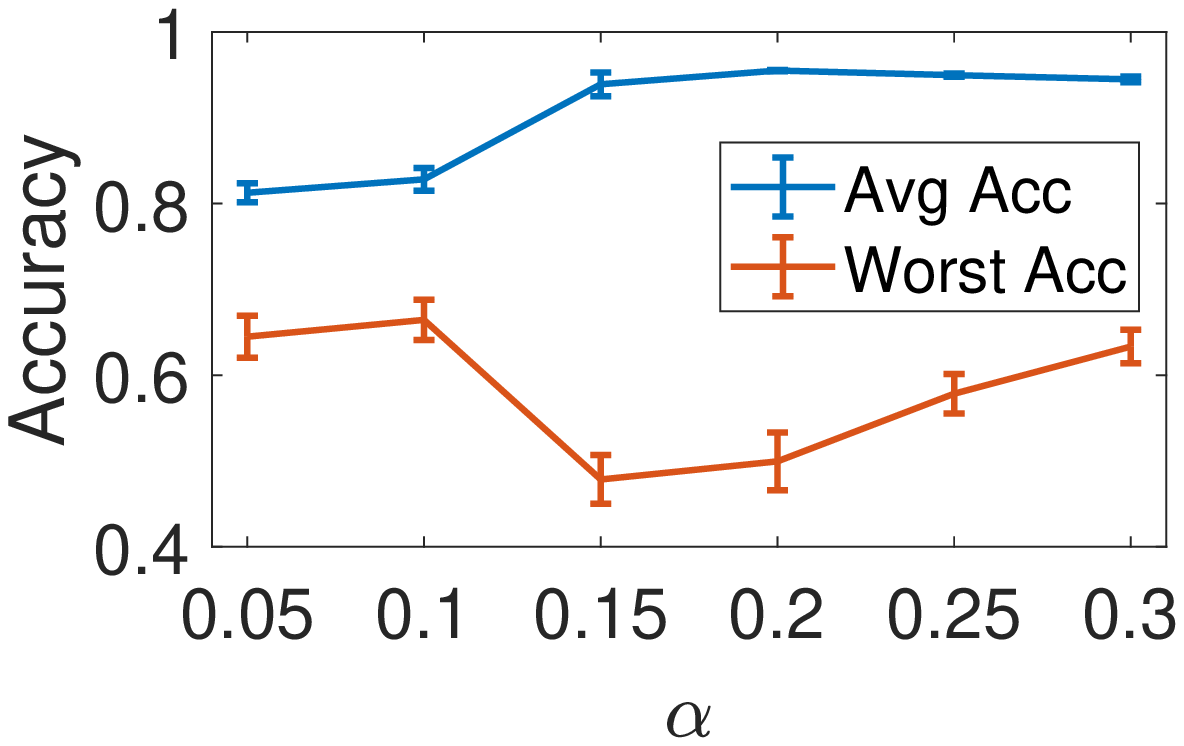}
         \caption{CVaR}
     \end{subfigure}
     \hfill
     \begin{subfigure}[b]{0.48\linewidth}
         \centering
         \includegraphics[width=\textwidth]{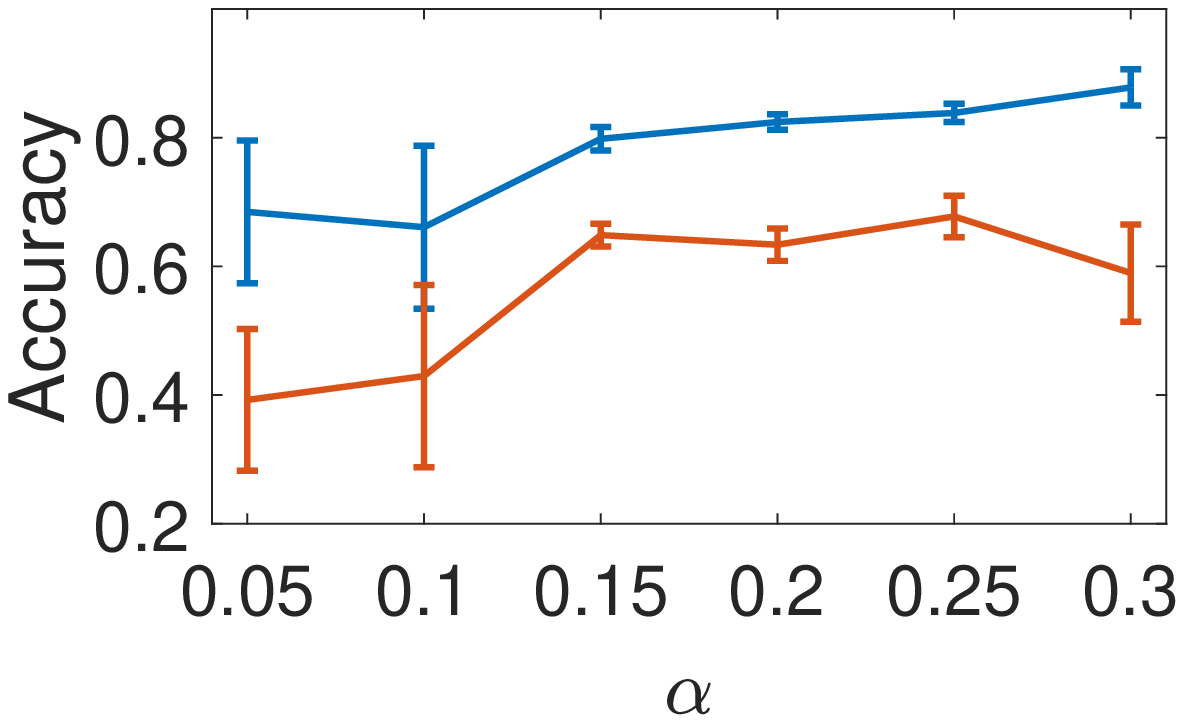}
         \caption{$\chi^2$-DRO}
     \end{subfigure}
     \begin{subfigure}[b]{0.48\linewidth}
         \centering
         \includegraphics[width=\textwidth]{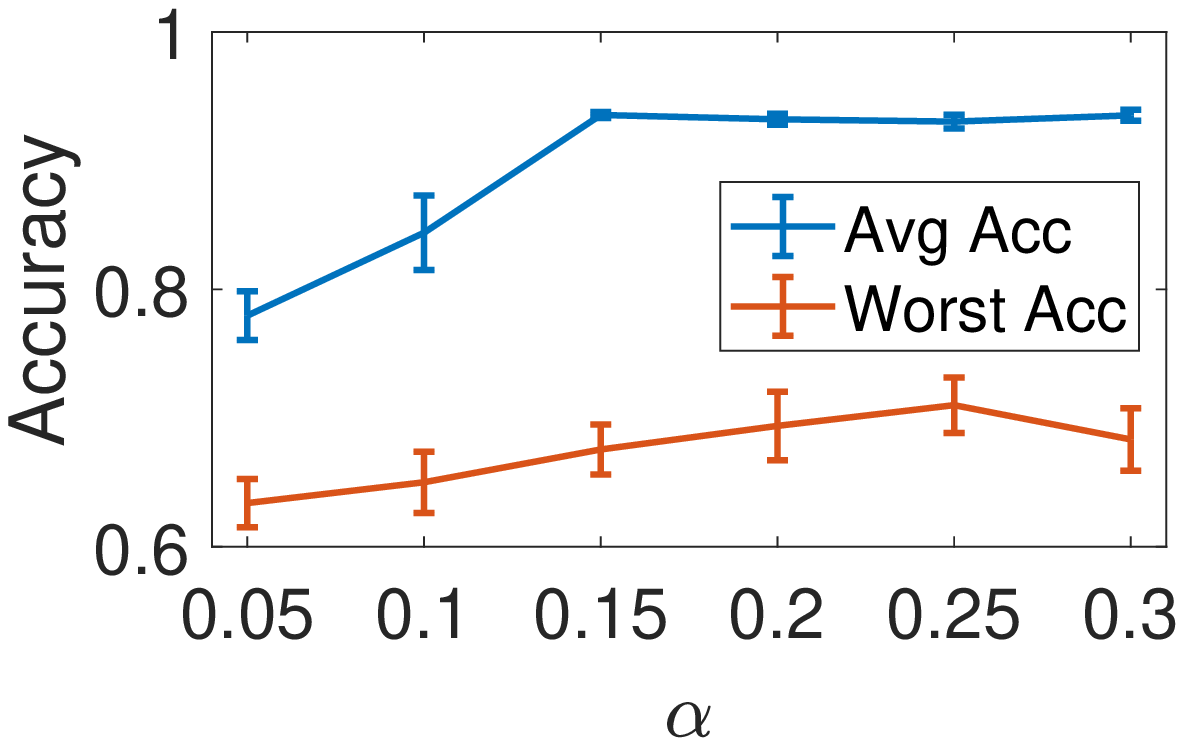}
         \caption{CVaR-DORO}
     \end{subfigure}
     \hfill
     \begin{subfigure}[b]{0.48\linewidth}
         \centering
         \includegraphics[width=\textwidth]{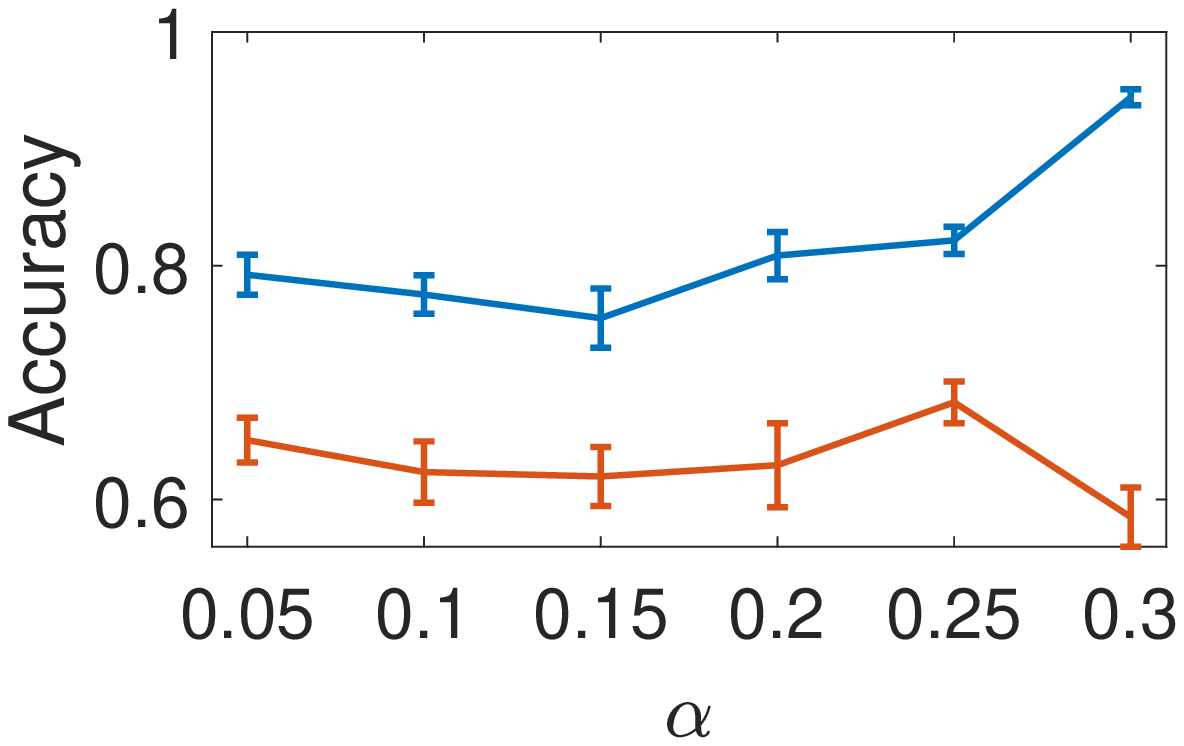}
         \caption{$\chi^2$-DORO}
     \end{subfigure}
    \caption{Effect of $\alpha$ on the test accuracies of DRO and DORO on CelebA ($\epsilon=0.01$).}
    \label{fig:alpha-celeba}
    \vskip -0.1in
\end{figure}

In this part, we study how $\alpha$ and $\epsilon$ affect the test accuracies of DORO with two experiments on CelebA, providing insight into how to select the optimal hyperparameters.

In the first experiment, we fix $\alpha=0.2$, and run the two DORO algorithms with different values of $\epsilon$. The results are plotted in Figure \ref{fig:eps-celeba}. We can see that for both methods, as $\epsilon$ increases, the average accuracy slightly decreases, while the worst-case accuracy first rises and then drops. Both average and worst-case accuracies will drop if $\epsilon$ is too big. Moreover, both methods achieve the optimal worst-case accuracy at $\epsilon=0.005$. We conjecture that the real noise level of the CelebA dataset is around 0.005, and that the optimal $\epsilon$ should be close to the real noise level.

In the second experiment, we run DRO and DORO ($\epsilon=0.01$) with different values of $\alpha$. The results are plotted in Figure \ref{fig:alpha-celeba}. First, we observe that for all methods, the optimal $\alpha$ is much bigger than the real $\alpha$ of the dataset. The real $\alpha$ of the CelebA dataset is around 0.008 (see Appendix \ref{app:domain-def}, Table \ref{tab:celeba-domains}), much smaller than those achieving the highest worst-case accuracies in the figures. Second, in all four figures the overall trend of the average accuracy is that it grows with $\alpha$. Third, both CVaR-DORO and $\chi^2$-DORO achieve the optimal worst-case accuracy at $\alpha=0.25$, but the worst-case accuracy drops as $\alpha$ goes to 0.3.

\section{Discussion}

In this work we pinpointed one direct cause of the performance drop and instability of DRO: the sensitivity of DRO to outliers in the dataset. We proposed DORO as an outlier robust refinement of DRO, and implemented DORO for the Cressie-Read family of R\'enyi divergence. We made a positive response to the open question raised by \cite{pmlr-v80-hashimoto18a} by demonstrating the effectiveness of DORO both theoretically and empirically.

One alternative approach to making DRO robust to outliers is removing the outliers from the dataset via preprocessing. In Section \ref{sec:dro-not-robust} we used a simple version of iterative trimming \cite{shen2019learning} to remove outliers from the training set. Compared to iterative trimming, DORO does not require retraining the model and does not throw away any data. In addition, preprocessing methods such as iterative trimming cannot cope with online data (where new instances are received sequentially), but DORO is still feasible.

The high-level idea of DORO can be extended to other algorithms that deal with subpopulation shift, such as static reweighting \cite{shimodaira2000improving}, adversarial reweighting \cite{hu2018does,lahoti2020fairness} and group DRO \cite{sagawa2019distributionally}. The implementations might be different, but the basic ideas are the same: to prevent the algorithm from overfitting to potential outliers. We leave the design of such algorithms to future work.

There is one large open question from this work. In our experiments, we found that model selection without domain information in the validation set is very hard. In Appendix \ref{app:model-select} we study several strategies, such as selecting the model with the lowest CVaR risk or the lowest CVaR-DORO risk, but none of them is satisfactory. A recent paper \cite{michel2021modeling} proposed two selection methods Minmax and Greedy-Minmax, but their performances are still much lower than the oracle's (see their Table 2a). \cite{gulrajani2021in} also pointed out the difficulty of model selection in domain-oblivious distributional shift tasks. Thus, we believe this question to be fairly non-trivial.

\camready{
\section*{Acknowledgements}

We acknowledge the support of DARPA via HR00112020006, and NSF via IIS-1909816, OAC-1934584.
}

\newpage
\onecolumn
\appendix
\section{Proofs}



\subsection{Proof of Proposition \ref{prop:dro-fwod-generalized}}
\label{proof:dro-fwod-generalized}

Since $P_k$ is a mixture component of $P$ with probability mass at least $\alpha$, we can see that 
\begin{equation}
\frac{d P_k}{d P} \leq \frac{1}{\alpha}
\end{equation}
Notice that when $t \ge 1$, \begin{equation}
f_\beta'(t) = \frac{1}{\beta - 1}(t^{\beta - 1} -1) >0
\end{equation}
Hence, $f_\beta'(t)$ is an increasing function when $t >1$, therefore, 
\begin{align}
D_\beta(P_k || P)  &=\int  f_\beta(\frac{dP_k}{dP}) dP \\
& \leq \int  f_\beta(\frac{1}{\alpha}) dP  \\
& = f_\beta(\frac{1}{\alpha})
\end{align}
Therefore, by the definition of $\beta$-DRO risk, we have completed the proof. \qed

\subsection{Proof of Corollary \ref{prop:dro-fwod}}
\label{proof:dro-fwod}
For any $k$, let $p_k = P(\gD_k)$, then $P(z)=p_kP(z|\gD_k)+(1-p_k)P(z|\overline{\gD_k})$ holds for all $x$. Let $Q=P_k$ and $Q'(z) = \frac{p_k-\alpha}{1-\alpha}P(z|\gD_k) + \frac{1-p_k}{1-\alpha}P(z|\overline{\gD_k})$. Then $P = \alpha Q + (1-\alpha) Q'$, which implies that $\E_{P_k}[\ell(\theta;Z)] \leq \cvar_\alpha(\theta;P)$. Thus, $\gR_{\max}(\theta;P) \leq \cvar_\alpha(\theta;P)$. On the other hand, for any $Q$ such that there exists $Q'$ satisfying $P=\alpha Q + (1-\alpha) Q'$, we have $\frac{dQ}{dP}(z) \leq \frac{1}{\alpha}$ a.e., so that $D_{\chi^2}(Q \parallel P) \leq \frac{1}{2}(\frac{1}{\alpha}-1)^2 = \rho$. Thus, $\cvar_{\alpha}(\theta;P) \leq \gR_{D_{\chi^2},  \rho}(\theta; P)$. \qed



\subsection{Proposition \ref{prop:robust-dro-cressie}}

\subsubsection{Proof of Proposition \ref{prop:robust-dro-cressie}}
\label{proof:prop-robust-dro-cressie}

By (\ref{eqn:dro-fwod-generalized}) and (\ref{eqn:robust-dro-risk}) we have
\begin{equation}
\label{eqn:prop3-chisq}
\begin{aligned} 
\gR_{D_\beta,\rho,\epsilon}(\theta;\ptrain) &= \inf_{P'}\left \lbrace \gR_{D_\beta,\rho}(\theta;P') : \exists \tilde{P'} \text{ s.t. } \ptrain = (1-\epsilon) P' + \epsilon \tilde{P'} \right \rbrace \\ 
&= \inf_{P',\eta} \left \lbrace c_\beta(\rho) \E_{P'} [(\ell(\theta;Z) - \eta)_+^{\beta_*}]^{\frac{1}{\beta_*}} + \eta \right \rbrace \\ 
&= \inf_{\eta} \left \lbrace c_\beta(\rho) \inf_{P'} \{ [ \int_{\R_+} P'( (\ell(\theta;Z)-\eta)_+^{\beta_*} >u)du]^{\frac{1}{\beta_*}} \} + \eta \right \rbrace
\end{aligned}
\end{equation}

By $\ptrain = (1-\epsilon) P' + \epsilon \tilde{P'}$ we have for all $\ell_0$,
\begin{equation}
\label{eqn:prop2-l0}
P'(\ell(\theta;Z) \leq \ell_0) \leq \min \left \lbrace 1, \frac{1}{1-\epsilon}\ptrain (\ell(\theta;Z) \leq \ell_0) \right \rbrace
\end{equation}
and we can also show that there exists a $P^* = P'$ such that the equality is achieved in (\ref{eqn:prop2-l0}) for all $\ell_0$: Since both $\ell$ and $\ptrain$ are continuous, $\ptrain(\ell(\theta;z))$ is a continuous function of $z$ for any fixed $\theta$, so there exists an $\ell^*$ such that $\ptrain(\ell(\theta;Z) > \ell^*) = \epsilon$. Define
\begin{equation}
\label{eqn:prop2-pstar-1}
P^*(z) = \left \{
\begin{array}{cc}
\frac{1}{1-\epsilon} \ptrain(z) & ,\ell(\theta;z) \leq \ell^* \\ 
0 &, \ell(\theta;z) > \ell^*
\end{array}
\right .
\end{equation}


For (\ref{eqn:prop2-pstar-1}), we have $\int_{\gX \times \gY}P^*(z)dz = \frac{1}{1-\epsilon}\int_{\ell(\theta;z) < \ell^*} \ptrain(z)dz = \frac{1}{1-\epsilon}\ptrain(\ell(\theta;Z) < \ell^*) = 1$ because $\ptrain(\ell(\theta;Z)=\ell^*)=0$, so (\ref{eqn:prop2-pstar-1}) is a distribution function. 

Let $v=u^{\frac{1}{\beta_*}}$. Plugging $P^*(\ell(\theta;Z) \leq \ell_0) = \min \left \lbrace 1, \frac{1}{1-\epsilon}\ptrain (\ell(\theta;Z) \leq \ell_0) \right \rbrace$ into (\ref{eqn:prop3-chisq}) produces
\begin{equation}
\label{eqn:prop4-expand}
\begin{aligned} 
\gR_{D_\beta,\rho,\epsilon}(\theta;\ptrain) &= \inf_{\eta} \left \lbrace c_\beta(\rho) \left [ \int_{\R_+} [1-P^*( (\ell(\theta;Z)-\eta)_+^{\beta_*} \leq v^{\beta_*})]dv^{\beta_*} \right ]^{\frac{1}{\beta_*}}  + \eta \right \rbrace \\
&= \inf_\eta \left \lbrace c_\beta(\rho) \left [ \int_{\R_+} [1-\frac{1}{1-\epsilon}\ptrain (\ell(\theta;Z) \leq \eta+v)]_+ dv^{\beta_*} \right]^{\frac{1}{\beta_*}} + \eta \right \rbrace \\ 
&= \inf_\eta \left \lbrace  c_\beta(\rho)  \left [\int_{0}^{(\ell^*-\eta)_+} \frac{1}{1-\epsilon}[(1-\epsilon) - \ptrain (\ell(\theta;Z) \leq \eta+v)]_+ dv^{\beta_*} \right ]^{\frac{1}{\beta_*}} +\eta \right \rbrace
\end{aligned}
\end{equation}

On the other hand, we have
\begin{equation}
\begin{aligned}  
&\E_{\ptrain}[(\ell - \eta)_+^{\beta_*} \mid P_{Z' \sim \ptrain}(\ell(\theta;Z') > \ell(\theta;Z)) \geq \epsilon] \\
=& \frac{1}{1-\epsilon}\int_0^{\ell^*} (u - \eta)_+^{\beta_*} d(\ptrain(\ell \leq u)) \\ 
=& \frac{1}{1-\epsilon} \left \lbrace  \left [ (u-\eta)_+^{\beta_*} \ptrain(\ell \leq u) \right ]_{0}^{\ell^*} - \int_0^{\ell^*} \ptrain(\ell \leq u) d((u-\eta)_+^{\beta_*}) \right \rbrace \\
=& \frac{1}{1-\epsilon} \left \lbrace  (\ell^*-\eta)_+^{\beta_*}(1-\epsilon) - \int_0^{\ell^*} \ptrain(\ell \leq u) d((u-\eta)_+^{\beta_*}) \right \rbrace \\ 
=& \frac{1}{1-\epsilon} \left \lbrace \int_{0}^{(\ell^*-\eta)_+} (1-\epsilon) dv^{\beta_*} - \int_{0}^{(\ell^*-\eta)_+} \ptrain (\ell \leq \eta + w) dw^{\beta_*} \right \rbrace
\end{aligned} 
\end{equation}
where $w = (u-\eta)_+$. Thus, (\ref{eqn:prop4-expand}) is equal to the right-hand side of (\ref{eqn:robust-dro-cressie}). \qed

\camready{

\subsubsection{Extension to Arbitrary $\ptrain$}
\label{proof:extend-formula}
For any distribution $\ptrain$, we can obtain a similar but more complex formula (\ref{eqn:extend-formula}). For any $\ptrain$, there exists an $\ell^*$ such that $\ptrain(\ell(\theta;Z) > \ell^*) \le \epsilon$ and $\ptrain(\ell(\theta;Z) < \ell^*) \le 1 - \epsilon$. If $\ptrain(\ell(\theta;Z) = \ell^*) = 0$, then the proof above is still correct, so the formula is still (\ref{eqn:robust-dro-cressie}).

Now assume that $\ptrain(\ell(\theta;Z) = \ell^*) > 0$. Similar to (\ref{eqn:prop2-pstar-1}), define

\begin{equation}
P^*(z) = \left \{
\begin{array}{cc}
\frac{1}{1-\epsilon} \ptrain(z) & ,\ell(\theta;z) < \ell^* \\ 
\left[1 - \frac{1}{1-\epsilon} \ptrain(\ell(\theta;Z) < \ell^*) \right] / \ptrain(\ell(\theta;Z) = \ell^*) &, \ell(\theta;z) = \ell^* \\ 
0 &, \ell(\theta;z) > \ell^*
\end{array}
\right .
\end{equation}

Then we still have $P^*(\ell(\theta;Z) \leq \ell_0) = \min \left \lbrace 1, \frac{1}{1-\epsilon}\ptrain (\ell(\theta;Z) \leq \ell_0) \right \rbrace$, so (\ref{eqn:prop4-expand}) still holds. On the other hand, we have
\begin{equation}
\begin{aligned}
&E_{\ptrain}[(\ell - \eta)_+^{\beta_*} \mid P_{Z' \sim \ptrain}(\ell(\theta;Z') > \ell(\theta;Z)) > \epsilon]  \\ 
=& \frac{1}{\ptrain(\ell(\theta;Z) < \ell^*)} \left \lbrace \int_{0}^{(\ell^*-\eta)_+} (1-\epsilon) dv^{\beta_*} - \int_{0}^{(\ell^*-\eta)_+} \ptrain (\ell \leq \eta + w) dw^{\beta_*} \right \rbrace
\end{aligned}
\end{equation}

Thus, the formula becomes
\begin{equation}
\label{eqn:extend-formula}
\begin{aligned}
\gR_{D_\beta, \rho, \epsilon}(\theta;\ptrain) =& \inf_\eta \{ c_\beta(\rho) ( \frac{\ptrain(\ell < \ell^*)}{1-\epsilon} \E_{Z}[(\ell(\theta;Z) -\eta)_+^{\beta_*} \mid  P_{Z'} (\ell(\theta;Z')>\ell(\theta;Z)) > \epsilon] \\ 
&+\frac{1-\ptrain(\ell < \ell^*)}{1-\epsilon} (\ell^* - \eta)_+^{\beta_*} )^{\frac{1}{\beta_*} } + \eta \}  
\end{aligned}
\end{equation}

}

\subsection{Proofs of Results in Section \ref{sec:theory}}
\label{proof:sec-theory}

\subsubsection{A Key Technical Lemma}
The following lemma will be useful in the analysis of CVaR-DORO and $\chi^2$-DORO: it controls the difference of dual objective in two distributions $P, P'$ by their total variation distance, with the assumption that loss function $l$ has bounded $2k$-th moment under $P$.

\begin{lem}\label{lem:key-technical}
	For any distributions $P, P'$, non-negative loss function $l(\cdot, Z)$ and $ 1\le \beta_* <2k$, such that $\E_P[l(\theta, Z)^{2k}] <\infty$, we have
	\begin{equation}
	\E_{P}[(\ell - \eta)^{\beta_*}_+]^{\frac{1}{\beta_*}} \leq 	\E_{P'}[(\ell - \eta)^{\beta_*}_+]^{\frac{1}{\beta_*}} + \E_P[(l(\theta, Z)-\eta)_+^{2k}]^{\frac{1}{2k}} \tv(P, P')^{\left(\frac{1}{\beta_*} - \frac{1}{2k}\right)} \beta_*^{-\frac{1}{2k}} \cdot \left(\frac{2k}{2k-\beta_*}\right)^{\frac{1}{\beta_*}}
	\end{equation}
	%
\end{lem}
\begin{proof}
	By the definition of total variation distance, we have
	\begin{equation}
	\label{eqn:tv-loss-bound}
	P(\ell(\theta;Z) > u) - P'(\ell(\theta;Z') > u) \leq  \tv(P,P')
	\end{equation}
	holds for any $u \geq 0$.	
	
	By Markov's Inequality and the non-negativity of $\ell$, we have for any $\eta \geq 0$,
	\begin{equation} \label{eqn:markov}
	P(\ell - \eta > u) \le  \frac{\E[(\ell - \eta)_+^{2k}]}{u^{2k}} := (\frac{s_{2k}}{u})^{2k} 
	\end{equation}
	where we introduced the shorthand $s_{2k} := \E[(\ell - \eta)_+^{2k}]^{\frac{1}{2k}}$
	Using integration by parts, we can see that:
	\begin{align}
	\E_{P}[(\ell - \eta)^{\beta_*}_+] &= \int_\eta^\infty \beta_* (t - \eta)^{(\beta_*-1)} P(\ell \ge t) dt \\
	&= \int_0^\infty \beta_* u^{(\beta_*-1)} P(\ell -\eta \ge u) du
	\end{align} 
	Thus,
	\begin{equation}
	\label{eqn:thm6-ineq}
	\begin{aligned} 
	\E_{P}[(\ell - \eta)^{\beta_*}_+] - \E_{P'}[(\ell - \eta)^{\beta_*}_+] &=  \int_0^\infty \beta_* u^{(\beta_*-1)} \left( P(\ell -\eta \ge u) - P'(\ell -\eta \ge u)\right) du \\
	&=  \left(\int_0^M +\int_M^\infty\right)  \left(\beta_* u^{(\beta_*-1)} \left( P(\ell -\eta \ge u) - P'(\ell -\eta \ge u)\right) du \right)
	\end{aligned} 
	\end{equation}
	Here, $M$ is a positive parameter whose value will be determined later. Next, we will upper bound each of the two integrals separately.
	By \eqref{eqn:thm6-ineq}, 
	\begin{align}
	\int_0^M \beta_* u^{(\beta_*-1)} \left( P(\ell -\eta \ge u) - P'(\ell -\eta \ge u)\right) du  &\leq \int_0^M \beta_* u^{(\beta_*-1)} \tv(P, P')du \\
	&= M^{\beta_*} \tv(P, P'),
	\end{align}
	which gives an upper bound for the first integral. For the second integral, notice that $P'(\ell -\eta \ge u)$ is non-negative and use \eqref{eqn:markov}, we have:
	\begin{align}
	\int_M^{\infty} \beta_* u^{(\beta_*-1)} \left( P(\ell -\eta \ge u) - P'(\ell -\eta \ge u)\right) du  &\leq \int_M^{\infty} \beta_* u^{(\beta_*-1)} P(\ell -\eta \ge u)du \\
	&\leq  \int_M^{\infty} \beta_* u^{(\beta_*-1)} \left( \frac{s_{2k}}{u}\right)^{2k} \\
	&= \frac{s_{2k}^{2k}}{2k - \beta_*} \cdot \frac{1}{M^{2k-\beta_*}}
	\end{align}
	Therefore, by setting $M =s_{2k} (\tv(P, P') \beta_*)^{-1/2k}$ which minimizes the sum of two terms, we have
	\begin{equation}
	\E_{P}[(\ell - \eta)^{\beta_*}_+] - \E_{P'}[(\ell - \eta)^{\beta_*}_+] \leq \inf_{M>0} \left(M^{\beta_*} \tv(P, P') + \frac{s_{2k}^{2k}}{2k - \beta_*} \cdot \frac{1}{M^{2k-\beta_*}}\right) = s_{2k}^{\beta_*} \tv(P, P')^{1 - \frac{\beta_*}{2k}} \beta_*^{-\frac{\beta_*}{2k}} \cdot \frac{2k}{2k-\beta_*}
	\end{equation}
	
	Using the inequality $(A+B)^{\frac{1}{\beta_*}} \leq A^{\frac{1}{\beta_*}} + B^{\frac{1}{\beta_*}}$ when $\beta_* \geq 1$, we have:
	\begin{equation}
	\label{eqn:general-moment-ineq}
	\E_{P}[(\ell - \eta)^{\beta_*}_+]^{\frac{1}{\beta_*}} \leq 	\E_{P'}[(\ell - \eta)^{\beta_*}_+]^{\frac{1}{\beta_*}} + s_{2k} \tv(P, P')^{\left(\frac{1}{\beta_*} - \frac{1}{2k}\right)} \beta_*^{-\frac{1}{2k}} \cdot \left(\frac{2k}{2k-\beta_*}\right)^{\frac{1}{\beta_*}}
	\end{equation}
\end{proof}
\subsubsection{Proof of Lemma \ref{lem:robust-dro}}
For any $P'$ such that $\ptrain = (1-\epsilon) P' + \epsilon \tilde{P'}$ for some $\tilde{P'}$, let $U = P \land P'$, i.e. $U(z) = \min \{P(z), P'(z)\}$ for any $z \in \gX \times \gY$. We have
\begin{equation}
(1-\epsilon) U(z) + \epsilon \tilde{P}(z) + \epsilon \tilde{P'}(z) \geq \ptrain(z) \quad \text{for any } z \in \gX \times \gY
\end{equation}
because both $\tilde{P}(z)$ and $\tilde{P'}(z)$ are non-negative. Integrating both sides produces
\begin{equation}
\int_{\gX \times \gY} U(z) dz \geq \frac{1-2\epsilon}{1-\epsilon}
\end{equation}
which implies $\tv (P,P') \leq \frac{\epsilon}{1-\epsilon}$. Thus,
\begin{equation}
\gR_{D,\rho}(\theta;P') \geq \inf_{P''}\{ \gR_{D,\rho}(\theta,P''):  \tv (P, P'') \leq \frac{\epsilon}{1-\epsilon} \}
\end{equation}
which together with (\ref{eqn:robust-dro-risk}) proves (\ref{eqn:lem1}). \qed

\subsubsection{Proof of Theorem \ref{thm:thm5}, Analysis of CVaR-DORO}
\begin{proof}[Proof of Theorem \ref{thm:thm5}, CVaR-DORO]
	For any $\theta$, by Lemma \ref{lem:robust-dro} we have
	\begin{equation}
	\label{eqn:thm5-1}
	\cvar_{\alpha,\epsilon}(\theta;\ptrain) \geq \cvar_{\alpha,\epsilon}(\hat{\theta};\ptrain) \geq \inf_{P'}\{ \cvar_\alpha (\hat{\theta};P'): \tv(P,P') \leq \frac{\epsilon}{1-\epsilon} \}
	\end{equation}
	By Lemma \ref{lem:key-technical}, we have for any $\eta \ge 0$ and $\tv(P,P') \leq \frac{\epsilon}{1-\epsilon}$, 	
	\begin{equation}
	\label{eqn:thm6-ineq2}
	\begin{aligned} 
	\E_{P}[(\ell - \eta)_+] - \E_{P'}[(\ell - \eta)_+] & \leq  \left(1+\frac{1}{2k-1}\right) \E_P[(\ell - \eta)_+^{2k}]^{\frac{1}{2k}} \tv(P,P')^{1 - \frac{1}{2k}} \\
	&\leq \left(1+\frac{1}{2k-1}\right) \sigma_{2k} \tv(P,P')^{1 - \frac{1}{2k}}
	\end{aligned} 
	\end{equation}
	Here, we used the fact that $0\leq (\ell - \eta)_+^{2k} \leq \ell^{2k} $
	Moreover, for any $\eta < 0$, $\E_{P}[(\ell - \eta)_+] - \E_{P'}[(\ell - \eta)_+] = \E_{P}[(\ell - 0)_+] - \E_{P'}[(\ell - 0)_+]$ because $\ell$ is non-negative. So (\ref{eqn:thm6-ineq2}) holds for all $\eta \in \R$. Thus, by (\ref{eqn:cvar-dual}) we have for any $\eta \in \R$,
	\begin{equation}
	\cvar_{\alpha}(\hat{\theta};P) \leq \alpha^{-1}\E_P [(\ell-\eta)_+] + \eta \leq \alpha^{-1} \left\{ \E_{P'}[(\ell-\eta)_+] + \left(1+\frac{1}{2k-1}\right) \left(\frac{\epsilon}{1-\epsilon}\right)^{1 - \frac{1}{2k}} \right\} + \eta
	\end{equation}
	and taking the infimum over $\eta$, we have the following inequality holds for any $\theta$:
	\begin{equation}
	\cvar_{\alpha}(\hat{\theta};P) \leq \cvar_\alpha (\hat{\theta};P') + \left(1+\frac{1}{2k-1}\right)  \alpha^{-1} \sigma_{2k} \left(\frac{\epsilon}{1-\epsilon}\right)^{1 - \frac{1}{2k}} 
	\end{equation}
	
	By (\ref{eqn:robust-dro-risk}) we have $\cvar_{\alpha,\epsilon}(\theta;\ptrain) \leq \cvar_\alpha (\theta;P)$. Thus, by (\ref{eqn:thm5-1}), taking the infimum over $P'$ yields
	\begin{align}
	\cvar_{\alpha}(\hat{\theta};P) &\leq \cvar_{\alpha,\epsilon}(\theta;\ptrain) + \left(1+\frac{1}{2k-1}\right) \alpha^{-1} \sigma_{2k} \left(\frac{\epsilon}{1-\epsilon}\right)^{1 - \frac{1}{2k}} \\
	&\leq \cvar_\alpha (\theta;P) + \left(1+\frac{1}{2k-1}\right) \alpha^{-1} \sigma_{2k} \left(\frac{\epsilon}{1-\epsilon}\right)^{1 - \frac{1}{2k}}
	\end{align}
	
	Taking the infimum over $\theta$ completes the proof. 
\end{proof}

\subsubsection{Proof of Theorem \ref{thm:thm5}, Analysis of $\chi^2$-DORO}
We begin with a structral lemma about the optimal dual variable $\eta$ in the dual formulation \eqref{eqn:dro_dual}. Recall that $\beta = \beta_* = 2$ for $\chi^2$ divergence.
\begin{lem}\label{lem:chi2_structral}
	Let $\eta^*(P)$ be the minimizer of \eqref{eqn:dro_dual}. We have the following characterization about $\eta^*(P)$:
	\begin{enumerate}
		\item When $\rho \leq \frac{\Var_P[l(\theta, Z)]}{2\E[l(\theta, Z)]^2}$, we have $\eta^* \leq 0$; \\
		Furthermore, the DRO risk and optimal dual variable $\eta^*$ can be formulated as:
		\begin{align}
		\gR_{D_{\chi^2},\rho}(\theta;P) &= \E_P[l(\theta, Z)] + \sqrt{2\rho \Var_P[l(\theta, Z)]} \\
		\eta^* &= \E_P[l(\theta, Z)] - \sqrt{\frac{ \Var_P[l(\theta, Z)]}{2\rho}}
		\end{align}
		\item When $\rho \ge \frac{\Var_P[l(\theta, Z)]}{2\E[l(\theta, Z)]^2}$, we have $\eta^* \ge 0$.
	\end{enumerate}
\end{lem}
\begin{proof}
	(1) We will prove that for any $\rho >0$,
	\begin{equation}\label{eqn:variance_regularization}
	\gR_{D_{\chi^2},\rho}(\theta;P) \leq \E_P[l(\theta, Z)] + \sqrt{2\rho \Var_P[l(\theta, Z)]}
	\end{equation}
	and the equality is achievable when $\rho \leq \frac{\Var_P[l(\theta, Z)]}{2\E[l(\theta, Z)]^2}$.
	
	By the definition of $\chi^2$-DRO risk, 
	\begin{equation}
	\gR_{D_{\chi^2},\rho}(\theta;P) = \sup_{Q: D_{\chi^2}(Q||P) \leq \rho} E_Q [l(\theta, Z)]
	\end{equation}
	Let $\mu := \E_P [l(\theta, Z)]$, notice that
	\begin{align} \label{eqn:ub_EQL}
	\E_Q [l(\theta, Z)] &= \E_P [l(\theta, Z) \frac{dQ}{dP}] \\
	&= \E_P [l(\theta, Z)] + \E_P [l(\theta, Z) \left(\frac{dQ}{dP}-1\right)]\\
	&= \mu + \E_P [\left(l(\theta, Z)-\mu\right) \left(\frac{dQ}{dP}-1\right)]
	\end{align}
	where in the last step we used the fact that $E_P \frac{dQ}{dP} = 1$. 
	
	By the definition of $\chi^2$ divergence,  
	\begin{equation}
	\E_P [\left(\frac{dQ}{dP}-1\right)^2] = 2  D_{\chi^2}(Q||P) \leq 2 \rho,
	\end{equation}
	Therefore, by Cauchy-Schwarz inequality,
	\begin{align}
	\E_P [\left(l(\theta, Z)-\mu\right) \left(\frac{dQ}{dP}-1\right)] &\leq \left( \E_P [\left(l(\theta, Z)-\mu\right)]\right)^{1/2} \left(\E_P [\left(\frac{dQ}{dP}-1\right)^2]\right)^{1/2} \\
	&\leq \sqrt{\Var_P[l(\theta, Z)] \cdot 2\rho}
	\end{align}
	Plug in this upper bound to \eqref{eqn:ub_EQL} completes the proof of \eqref{eqn:variance_regularization}.
	
	To see that the equality can be achieved when $\rho \leq \frac{\Var_P[l(\theta, Z)]}{\E[l(\theta, Z)]^2}$, we only need to verify that $\eta = \eta^*$ gives the same dual objective $ \E_P[l(\theta, Z)] + \sqrt{2\rho \Var_P[l(\theta, Z)]}$. Since $\eta^* <0$, we have 
	\begin{align}
	&\E_P[(l(\theta, Z)-\eta^*)_+^2] \\
		 =& \E_P[(l(\theta, Z)-\eta^*)^2]\\
		=& \E_P[(l(\theta, Z) - \E_P[l(\theta, Z)] + \sqrt{\frac{1}{2\rho} \Var_P[l(\theta, Z)]})^2] \\
		=& \E_P[(l(\theta, Z) - \E_P[l(\theta, Z)])^2]  + 2  \sqrt{\frac{1}{2\rho}\Var_P[l(\theta, Z)]} \E[(l(\theta, Z) - \E_P[l(\theta, Z) ])] + \frac{1}{2\rho} \Var_P[l(\theta, Z)] \\
		=& \Var_P[l(\theta, Z)] + 0 + + \frac{1}{2\rho} \Var_P[l(\theta, Z)] = \frac{1+2\rho}{2\rho} \Var_P[l(\theta, Z)]
	\end{align}
	Therefore, 
	\begin{align}
		\sqrt{1+2\rho}\left(\E_P[(l(\theta, Z)-\eta^*)_+^2]\right)^{1/2} + \eta^* &= \frac{1+2\rho}{\sqrt{2\rho}} \sqrt{\Var_P[l(\theta, Z)]} + \E_P[l(\theta, Z)] - \frac{1}{\sqrt{2\rho}} \sqrt{\Var_P[l(\theta, Z)]}\\
		&=  \E_P[l(\theta, Z)] + \sqrt{2\rho\Var_P[l(\theta, Z)] }
	\end{align}
	and we have completed the proof.
	
	(2) 
	Let $g(\eta, P) = \sqrt{1+2\rho} \left(\E_P[(l(\theta, Z) - \eta)_+^2]\right)^{\frac{1}{2}} + \eta$ and recall that $\gR_{D_{\chi^2},\rho}(\theta;P) = \inf_{\eta \in \R} g(\eta, P) $. To show that $\eta^* \leq 0$, we only need to prove that $g(\eta) \leq g(0)$ whenever $\eta <0$, which is equivalent to:
	\begin{equation}
	\sqrt{1+2\rho} \left(\E_P[(l(\theta, Z) - \eta)^2]\right)^{\frac{1}{2}}  \geq \sqrt{1+2\rho} \left(\E_P[l(\theta, Z)^2]\right)^{\frac{1}{2}} - \eta
	\end{equation}
	Since both sides are non-negative, this inequality is equivalent to:
	\begin{equation}
	(1+2\rho) \E_P[(l(\theta, Z) - \eta)^2 \geq (1+2\rho) \E_P[l(\theta, Z)^2 + \eta^2 - 2\eta \sqrt{1+2\rho} \left(\E_P[l(\theta, Z)^2]\right)^{\frac{1}{2}}
	\end{equation}
	After re-organizing terms, it remains to prove
	\begin{equation}\label{eqn:ineq_neg_eta}
	2\rho \eta^2 - 2(1+2\rho) \eta E_P[l(\theta, Z)] + 2\eta \sqrt{1+2\rho} \left(\E_P[l(\theta, Z)^2]\right)^{\frac{1}{2}} \geq 0
	\end{equation}
	Since $\rho \ge \frac{\Var_P[l(\theta, Z)]}{2\E[l(\theta, Z)]^2}$, we have $(1+2\rho) \ge \frac{\E[l(\theta, Z)^2]}{\E[l(\theta, Z)]^2}$. Therefore, 
	\begin{align}
	LHS &\geq  2\eta \sqrt{1+2\rho} \left(\E_P[l(\theta, Z)^2]\right)^{\frac{1}{2}}- 2(1+2\rho) \eta E_P[l(\theta, Z)] \\
	&= 2\eta \sqrt{1+2\rho} \left(\left(\E_P[l(\theta, Z)^2]\right)^{\frac{1}{2}} - \sqrt{1+2\rho}  E_P[l(\theta, Z)] \right) \\
	& \geq 0
	\end{align}
	where in the last step we used the assumption that $\eta \leq 0$. Therefore we have completed the proof.
	
\end{proof}
Having prepared with Lemma \ref{lem:chi2_structral}, we are now ready to prove the $\chi^2$-DORO part of Theorem \ref{thm:thm5}.
\begin{proof}[Proof of Theorem \ref{thm:thm5}, $\chi^2$-DORO]
	We will first show that 
	\begin{equation}\label{eqn:ineq_chi2_dro}
	\gR_{D_{\chi^2},\rho}(\hat{\theta};P) \leq  \gR_{D_{\chi^2},\rho}(\hat{\theta};P') + \sqrt{1+2\rho} (1+C_{\rho}) \sigma_{2k} \tv(P, P')^{\left(\frac{1}{2} - \frac{1}{2k}\right)} 2^{-\frac{1}{2k}} \cdot \left(\frac{k}{k-1}\right)^{\frac{1}{2}}
	\end{equation}
	This inequality will be proved by combining two different strategies: when $\eta^*(P')$ is relatively large, we will use an argument based on Lemma \ref{lem:key-technical}, similar to what we did in the analysis of CVaR-DORO. Otherwise, when $\eta^*(P')$ is small, we need a different proof which builds upon the structral result Lemma \ref{lem:chi2_structral}.
	
	Define $C_{\rho} = \frac{\sqrt{1+2\rho}}{2\rho}$. Below we discuss two cases: $\eta^*(P') < -C_{\rho} \sigma_{2k}$ and $\eta^*(P') \ge -C_{\rho}  \sigma_{2k}$.
	
	\paragraph{Case 1: $\eta^*(P') < -C_{\rho} \sigma_{2k}$.}
	
	When $\eta^*(P') < -C_{\rho} \sigma_{2k}$, by Lemma \ref{lem:chi2_structral}, we have
	\begin{align}
	\gR_{D_{\chi^2},\rho}(\hat{\theta};P') &= \E_{P'}[l(\hat{\theta}, Z)] + \sqrt{2\rho \Var_{P'}[l(\hat{\theta}, Z)]} \\
	\eta^*(P') &= \E_{P'}[l(\hat{\theta}, Z)] - \sqrt{\frac{ \Var_{P'}[l(\hat{\theta}, Z)]}{2\rho}} < -C_{\rho} \sigma_{2k}
	\end{align}
	
	Therefore, we can lower bound $\sqrt{ \Var_{P'}[l(\hat{\theta}, Z)]}$ as:
	\begin{equation}
	\sqrt{ \Var_{P'}[l(\hat{\theta}, Z)]} \geq \sqrt{2\rho} \E_{P'}[l(\hat{\theta}, Z)] + \sqrt{2\rho} C_{\rho} \sigma_{2k} \ge \sqrt{2\rho} C_{\rho} \sigma_{2k},
	\end{equation}
	and consequently, we have a lower bound for $\gR_{D_{\chi^2},\rho}(\theta;P')$:
	\begin{align}\label{eqn:lb_p'}
	\gR_{D_{\chi^2},\rho}(\hat{\theta};P') &= \E_{P'}[l(\hat{\theta}, Z)] + \sqrt{2\rho \Var_{P'}[l(\hat{\theta}, Z)]} \\
	&\ge \sqrt{2\rho \Var_{P'}[l(\hat{\theta}, Z)]}  \ge 2 \rho C_{\rho} \sigma_{2k} = \sqrt{1+2\rho}\sigma_{2k}
	\end{align}
	On the other hand, by setting the dual variable $\eta = 0$, we have a simple upper bound for $\gR_{D_{\chi^2},\rho}(\hat{\theta};P)$:
	\begin{equation}\label{eqn:ub_p}
	\gR_{D_{\chi^2},\rho}(\hat{\theta};P) \leq \sqrt{1+2\rho} \E_{P}[l(\hat{\theta}, Z)^2]^{1/2} \leq \sqrt{1+2\rho} \sigma_{2k}
	\end{equation}
	Combining \eqref{eqn:lb_p'} and \eqref{eqn:ub_p}, we conclude that $ \gR_{D_{\chi^2},\rho}(\hat{\theta};P') \ge \gR_{D_{\chi^2},\rho}(\hat{\theta};P) $ and the inequality is trivially true.
	
	\paragraph{Case 2: $\eta^*(P') \ge -C_{\rho} \sigma_{2k}$}
	
	By Lemma \ref{lem:key-technical}, we have
	\begin{equation}
	\E_{P}[(\ell - \eta)^{2}_+]^{\frac{1}{2}} \leq 	\E_{P'}[(\ell - \eta)^{2}_+]^{\frac{1}{2}} + \E_Z[(l(\theta, Z)-\eta)_+^{2k}]^{\frac{1}{2k}} \tv(P, P')^{\left(\frac{1}{2} - \frac{1}{2k}\right)} 2^{-\frac{1}{2k}} \cdot \left(\frac{k}{k-1}\right)^{\frac{1}{2}}
	\end{equation}
	holds for any $\eta \in \R$. Since $\eta^*(P') \ge-C_{\rho} \sigma_{2k} $, we can upper bound the $2k$-th moment $E_Z[(l(\theta, Z)-\eta^*(P'))_+^{2k}]^{\frac{1}{2k}}$ as:
	\begin{align}
	E_Z[(l(\theta, Z)-\eta^*(P') )_+^{2k}]^{\frac{1}{2k}} &\leq E_Z[(l(\theta, Z)+ C_{\rho} \sigma_{2k})_+^{2k}]^{\frac{1}{2k}} \\
	&\leq E_Z[(l(\theta, Z)]^{\frac{1}{2k}}+ C_{\rho} \sigma_{2k} = (1+C_{\rho}) \sigma_{2k}
	\end{align}
	Hence, 
	\begin{align}
	\gR_{D_{\chi^2},\rho}(\hat{\theta};P) &\leq \sqrt{1+2\rho} \E_{P}[(\ell - \eta^*(P'))^{2}_+]^{\frac{1}{2}} + \eta^*(P') \\
	&\leq \sqrt{1+2\rho} \E_{P'}[(\ell - \eta^*(P'))^{2}_+]^{\frac{1}{2}} + \eta^*(P') + \sqrt{1+2\rho} (1+C_{\rho}) \sigma_{2k} \tv(P, P')^{\left(\frac{1}{2} - \frac{1}{2k}\right)} 2^{-\frac{1}{2k}} \cdot \left(\frac{k}{k-1}\right)^{\frac{1}{2}} \\
	&= \gR_{D_{\chi^2},\rho}(\hat{\theta};P') + \sqrt{1+2\rho} (1+C_{\rho}) \sigma_{2k} \tv(P, P')^{\left(\frac{1}{2} - \frac{1}{2k}\right)} 2^{-\frac{1}{2k}} \cdot \left(\frac{k}{k-1}\right)^{\frac{1}{2}}
	\end{align}
	
	Hence, we have proved the inequality \eqref{eqn:ineq_chi2_dro}. The rest of proof mimics CVaR-DORO. For any $\theta$, by Lemma \ref{lem:robust-dro} we have
	\begin{equation}
	\gR_{D_{\chi^2},\rho, \varepsilon}(\theta;\ptrain) \geq \gR_{D_{\chi^2},\rho, \varepsilon} \geq \inf_{P'}\{ \gR_{D_{\chi^2},\rho} (\hat{\theta};P'): \tv(P,P') \leq \frac{\epsilon}{1-\epsilon} \}
	\end{equation}
	By (\ref{eqn:robust-dro-risk}) we have $\gR_{D_{\chi^2},\rho, \varepsilon}(\theta;\ptrain) \leq \gR_{D_{\chi^2},\rho}(\theta;P)$. Thus, by (\ref{eqn:thm5-1}), taking the infimum over $P'$ yields
	\begin{align}
	\gR_{D_{\chi^2},\rho}(\hat{\theta};P) &\leq \gR_{D_{\chi^2},\rho,\epsilon}(\theta;\ptrain) +  \sqrt{1+2\rho}(1+C_{\rho}) \sigma_{2k} \left( \frac{\epsilon}{1-\epsilon}\right)^{\left(\frac{1}{2} - \frac{1}{2k}\right)} 2^{-\frac{1}{2k}} \cdot \left(\frac{k}{k-1}\right)^{\frac{1}{2}}  \\
	&\leq \gR_{D_\beta,\rho} (\theta;P) + \sqrt{1+2\rho} (1+C_{\rho}) \sigma_{2k} \left( \frac{\epsilon}{1-\epsilon}\right)^{\left(\frac{1}{2} - \frac{1}{2k}\right)} 2^{-\frac{1}{2k}} \cdot \left(\frac{k}{k-1}\right)^{\frac{1}{2}}
	\end{align}
	
	Taking the infimum over $\theta$ completes the proof.
\end{proof}

\subsubsection{Proof of Theorem \ref{thm:main_LB}}
We consider an optimization problem with the parameter space restricted to only two possible values $\Theta = \left\{\theta_0, \theta_1 \right\}$. Our proof is constructive, which relies on the following distribution $P_{M, \Delta, \varepsilon}$:
\begin{align}
l(\theta_0, Z) = 0, \quad & l(\theta_1, Z) = \Delta \quad w.p. \quad (1-\varepsilon) \\
l(\theta_0, Z) = M, \quad & l(\theta_1, Z) = \Delta \quad w.p. \quad \varepsilon
\end{align}
here $M, \Delta$ are some non-negative parameters whose value to be determined later and the probability is taken over the randomness of $Z$.

We have the following characterization of CVaR and $\chi^2$-DRO risk:
\begin{lem}[DRO Risk of $P_{M, \Delta, \varepsilon}$] 
	\label{lem:DRO_discrete_distribution} Assume that $\alpha \ge \varepsilon$ and $1+2\rho \leq \frac{1}{\varepsilon}$, we have the following closed-form expressions for CVaR and $\chi^2$-DRO risk:
	\begin{align}
	\cvar_{\alpha}(\theta_0;P_{M, \Delta, \varepsilon}) &= \frac{M \varepsilon}{\alpha}\\
	\cvar_{\alpha}(\theta_1;P_{M, \Delta, \varepsilon}) &= \Delta
	\end{align}
	and 
	\begin{align}
	\gR_{D_{\chi^2},\rho}(\theta_0;P_{M, \Delta, \varepsilon}) &= M \varepsilon + M \sqrt{2\rho\varepsilon(1-\varepsilon)}\\
	\gR_{D_{\chi^2},\rho}(\theta_1;P_{M, \Delta, \varepsilon}) &= \Delta
	\end{align}
\end{lem}
\begin{proof}
	Since $l(\theta_1, Z) $ is always a constant $\Delta$, it's immediate to see $\cvar_{\alpha}(\theta_1;P_{M, \Delta, \varepsilon}) = \gR_{D_{\chi^2},\rho}(\theta_1;P_{M, \Delta, \varepsilon}) = \Delta$. Hence we only need to focus on $\theta_0$.
	
	By the dual formulation of DRO risk, we have $\cvar_{\alpha}(\theta_0;P_{M, \Delta, \varepsilon}) = \inf_{\eta \in \R} h(\eta)$ and $\gR_{D_{\chi^2},\rho}(\theta_0;P_{M, \Delta, \varepsilon}) = \inf_{\eta \in \R} g(\eta)$, where we use the shorthand $g(\eta)$ and $h(\eta)$ for
	\begin{align}
	g(\eta) &:= \sqrt{1+2\rho} \left(\E_P[(l(\theta, Z) - \eta)_+^2]\right)^{\frac{1}{2}} + \eta \\
	h(\eta) &= \frac{1}{\alpha}\E_P[(l(\theta, Z) - \eta)_+] + \eta
	\end{align}
	Direct calculation gives:
	\begin{equation}
	g(\eta) = \begin{cases}
	\sqrt{1+2\rho} \sqrt{(\eta - \epsilon M)^2 + \varepsilon(1-\varepsilon)M^2} + \eta, & \text{for } \eta <0 \\
	\sqrt{\varepsilon(1+2\rho)}(M- \eta) + \eta, & \text{for } 0 \leq \eta \leq M \\
	\eta, & \text{for } \eta >M
	\end{cases}
	\end{equation}
	and 
	\begin{equation}
	h(\eta) = \begin{cases}
	\frac{M \varepsilon - \eta}{\alpha} + \eta, & \text{for } \eta <0 \\
	\frac{\varepsilon (M-\eta)}{\alpha} +\eta , & \text{for } 0 \leq \eta \leq M \\
	\eta, & \text{for } \eta >M
	\end{cases}
	\end{equation}
	Therefore, when $\alpha \ge \varepsilon$ and $1+2\rho \le \frac{1}{\varepsilon}$, we have
	\begin{align}
		\cvar_{\alpha}(\theta_0;P_{M, \Delta, \varepsilon}) &= \inf_{\eta \in \R} h(\eta) &= h(0) &= \frac{M\varepsilon}{\alpha} \\
		\gR_{D_{\chi^2},\rho}(\theta_0;P_{M, \Delta, \varepsilon}) &= \inf_{\eta \in \R} g(\eta) &= g(\varepsilon M - \frac{M \sqrt{\varepsilon(1-\varepsilon)}}{\sqrt{2\rho}} ) &=  M \varepsilon + M \sqrt{2\rho\varepsilon(1-\varepsilon)}
	\end{align}
	and we have completed the proof.
\end{proof}

Equipped with Lemma \ref{lem:DRO_discrete_distribution}, we are now ready to prove the main lower bound Theorem \ref{thm:main_LB}.

\begin{proof}[Proof of Theorem \ref{thm:main_LB}]	
	Consider $\ptrain = P_{M, \Delta, \varepsilon}$. We have two different ways to decompose $\ptrain$ into mixture of two distributions:
	\begin{equation}
		\ptrain  = P_{M, \Delta, \varepsilon} = (1-\varepsilon) P_{M, \Delta, \varepsilon} + \varepsilon P_{M, \Delta, \varepsilon} = (1-\varepsilon) P_{0, \Delta, 0} + \varepsilon P_{M, \Delta, 0}
	\end{equation}
	 In other words, with only access to $\ptrain  = P_{M, \Delta, \varepsilon}$, the learner cannot distinguish the following two possibilities:
	 \begin{itemize}
	 	\item (a) The clean distribution is $P = P_{M, \Delta, \varepsilon} $, and the outlier distribution is $P' = P_{M, \Delta, \varepsilon}$.
	 	\item (b) The clean distribution is $Q = P_{0, \Delta, 1} $, and the outlier distribution is $Q' = P_{M, \Delta, 1}$.
	 \end{itemize}
 
 Furthermore, as long as $M \leq \sigma_{2k} \varepsilon^{-\frac{1}{2k}}$ and $\Delta \leq \sigma_{2k}$, both $P$ and $Q$ satisfy the bounded $2k$-th moment condition $\E[l(\theta, Z)^{2k}] \leq \sigma_{2k}^{2k}$. With our construction below, we can ensure that $\theta_1$ is $\Theta(\Delta)$-suboptimal under $P$, while $\theta_0$ is $\Theta(\Delta)$-suboptimal under $Q$. Therefore, in the worst case scenario, it's impossible for the learner to find a solution with $O(\Delta)$ sub-optimality gap under both $P$ and $Q$.
	
\paragraph{CVaR lower bound} Assume that $\alpha \ge \frac{1}{2} \varepsilon^{1-\frac{1}{2k}} $. Let $M = \sigma_{2k} \varepsilon^{- \frac{1}{2k}} , \Delta =  \sigma_{2k} \frac{\varepsilon^{1-\frac{1}{2k}}}{2\alpha} \leq \sigma_{2k} $. Recall that $P = P_{M, \Delta, \varepsilon} $, by Lemma \ref{lem:DRO_discrete_distribution}, we have:
 	\begin{align}
		\cvar_{\alpha}(\theta_0;P) &= \frac{M \varepsilon}{\alpha} = \frac{\sigma_{2k}}{\alpha}\varepsilon^{1 - \frac{1}{2k}} = 2 \Delta \\
		\cvar_{\alpha}(\theta_1;P) &= \Delta
 	\end{align}
 	Therefore, 
 	\begin{equation}\label{eqn:cvar_lb_ineq1}
 		\cvar_{\alpha}(\theta_0;P) - \inf_{\theta \in \Theta} \cvar_{\alpha}(\theta;P) = \Delta = \Omega(\frac{1}{\alpha} \sigma_{2k} \varepsilon^{1 - \frac{1}{2k}})
 	\end{equation}
 	For $Q = P_{0, \Delta, 1}$, both $l(\theta_0, Z)$ and $l(\theta_1, Z)$ are constants, and hence 
 	\begin{align}
 	\cvar_{\alpha}(\theta_0;Q) &= 0 \\
 	\cvar_{\alpha}(\theta_1;Q) &= \Delta
 	\end{align}
 	and 
 	\begin{equation}\label{eqn:cvar_lb_ineq2}
 	\cvar_{\alpha}(\theta_1;Q) - \inf_{\theta \in \Theta} \cvar_{\alpha}(\theta;Q) = \Delta = \Omega(\frac{1}{\alpha} \sigma_{2k} \varepsilon^{1 - \frac{1}{2k}})
 	\end{equation}
 	Combining \eqref{eqn:cvar_lb_ineq1} and \eqref{eqn:cvar_lb_ineq2} completes the proof.
 	
 \paragraph{$\chi^2$-DRO lower bound} Assume that $\rho = O(\varepsilon^{\frac{1}{k} -1})$. Let $M = \sigma_{2k} \varepsilon^{- \frac{1}{2k}} , \Delta =  \frac{M}{2} \left( \varepsilon +  \sqrt{2\rho\varepsilon(1-\varepsilon)}\right) \leq \sigma_{2k} $. Recall that $P = P_{M, \Delta, \varepsilon} $, by Lemma \ref{lem:DRO_discrete_distribution}, we have:
 \begin{align}
 \gR_{D_{\chi^2},\rho}(\theta_0;P) &= M \varepsilon + M \sqrt{2\rho\varepsilon(1-\varepsilon)} = 2 \Delta \\
 \gR_{D_{\chi^2},\rho}(\theta_1;P) &= \Delta
 \end{align}
 Therefore, 
 \begin{equation}\label{eqn:chi2_lb_ineq1}
 \gR_{D_{\chi^2},\rho}(\theta_0;P) - \inf_{\theta \in \Theta} \gR_{D_{\chi^2},\rho}(\theta;P) = \Delta = \Omega( \sigma_{2k} \sqrt{\rho}\varepsilon^{\frac{1}{2} - \frac{1}{2k}})
 \end{equation}
 For $Q = P_{0, \Delta, 1}$, both $l(\theta_0, Z)$ and $l(\theta_1, Z)$ are constants, and hence 
 \begin{align}
 \cvar_{\alpha}(\theta_0;Q) &= 0 \\
 \cvar_{\alpha}(\theta_1;Q) &= \Delta
 \end{align}
 and 
 \begin{equation}\label{eqn:chi2_lb_ineq2}
 \cvar_{\alpha}(\theta_0;Q) - \inf_{\theta \in \Theta} \cvar_{\alpha}(\theta;Q) = \Delta =  \Omega( \sigma_{2k} \sqrt{\rho}\varepsilon^{\frac{1}{2} - \frac{1}{2k}})
 \end{equation}
 Combining \eqref{eqn:chi2_lb_ineq1} and \eqref{eqn:chi2_lb_ineq2} completes the proof.	
 
\end{proof}

\subsubsection{Proof of Theorem \ref{thm:effectiveness}}

By Lemma \ref{lem:key-technical}, for any $P'$ such that $\tv (P,P') \leq \frac{\epsilon}{1-\epsilon}$, 
\begin{equation}
\cvar_{\alpha}(\theta;P) - \cvar_{\alpha}(\theta;P')	\leq 2\alpha^{-1}\sigma \sqrt{\frac{\epsilon}{1-\epsilon}}
\end{equation}
By Proposition \ref{prop:dro-fwod}, if $\gR_{\max}(\theta;P) > 3\alpha^{-1}\sigma\sqrt{\frac{\epsilon}{1-\epsilon}}$, then $\cvar_{\alpha}(\theta;P) > 3\alpha^{-1}\sigma\sqrt{\frac{\epsilon}{1-\epsilon}}$, which implies that
\begin{equation}
\label{eqn:cvar-lower-bound}
\frac{\cvar_\alpha(\theta;P')}{\gR_{\max}(\theta;P)} \geq \frac{\cvar_\alpha(\theta;P')}{\cvar_\alpha(\theta;P)} = 1 - \frac{\delta}{\cvar_\alpha(\theta;P)} \geq 1-\frac{2\alpha^{-1}\sigma\sqrt{\frac{\epsilon}{1-\epsilon}}}{3\alpha^{-1}\sigma\sqrt{\frac{\epsilon}{1-\epsilon}}} = \frac{1}{3}
\end{equation}
holds for any $P'$ such that $\tv (P,P') \leq \frac{\epsilon}{1-\epsilon}$. By Lemma \ref{lem:robust-dro}, taking the infimum over $P'$ yields the first inequality of (\ref{eqn:thm-effectiveness}). Moreover, by Proposition \ref{prop:dro-fwod}, for any $\theta$ and $P'$, $D_{\chi^2,\rho}(\theta;P') \geq \cvar_\alpha(\theta;P')$, which together with (\ref{eqn:cvar-lower-bound}) yields the second inequality of (\ref{eqn:thm-effectiveness}). \qed

\section{Experiment Details}
\label{app:exp-details}
\subsection{Domain Definition}
\label{app:domain-def}
One important decision we need to make when we design a task with subpopulation shift is how to define the domains (subpopulations). We refer our readers to the Wilds paper \cite{koh2020wilds}, which discusses in detail the desiderata and considerations of domain definition, and defines 16 domains on the CivilComments-Wilds dataset which we use directly. The authors selected 8 features such as race, sex and religion, and crossed them with the two classes to define the 16 domains. Such a definition naturally covers class imbalance. There is no official domain definition on CelebA, so we define the domains on our own. Following their approach, on CelebA we also select 8 features and cross them with the two classes to compose the 16 domains. Our definition is inspired by \cite{sagawa2019distributionally}, but we cover more types of subpopulation shift apart from demographic differences.

We select 8 features on CelebA: \textit{Male}, \textit{Female}, \textit{Young}, \textit{Old}, \textit{Attractive}, \textit{Not-attractive}, \textit{Straight-hair} and \textit{Wavy-hair}. We explain why we select these features as follows:
\begin{itemize}
	\item The first four features cover sex and age, two protected features widely used in algorithmic fairness papers.
	\item We select the next two features in order to cover labeling biases, biases induced by the labelers into the dataset. Among the 40 features provided by CelebA, the \textit{Attractive} feature is the most subjective one. Table \ref{tab:celeba-domains} shows that among the people with blond hair, more than half are labeled \textit{Attractive}; while among the other people, more than half are labeled \textit{Not-attractive}. It might be that the labelers consider blond more attractive than other hair colors, or it might be that the labelers consider females more attractive than males, and it turns out that more females have blond hair than males in this dataset. Although the reason behind is unknown, we believe that these two features well represent the labeling biases in this dataset, and should be taken into consideration.
	\item We select the last two features in order to cover confounding variables, features the model uses to do classification that should have no correlation with the target by prior knowledge. Since the target is the hair color, a convolutional network trained on this dataset would focus on the hair of the person, so we conjecture that the output of the convolutional network is highly correlated with the hair style. In our experiments, we find that models trained with ERM misclassify about 20\% of the test instances with blond straight hair, much more than the other three combinations.
\end{itemize}

\begin{table}[!t]
	\caption{Number of training instances in each domain of CelebA and CivilComments-Wilds.}
	\label{tab:celeba-domains}
	\vskip 0.15in
	\begin{subtable}[!t]{0.482\textwidth}
		\centering
		\begin{small}
			\begin{tabular}{ccc}
				\toprule
				\textbf{CelebA} & Blond & Others \\
				\midrule
				Male  & 1387 & 66874 \\
				Female & 22880 & 71629 \\ 
				Young & 20230 & 106558 \\ 
				Old & 4037 & 31945 \\ 
				Attractive & 17008 & 66595 \\ 
				Not-attractive & 7259 & 71908 \\ 
				Straight-hair & 5178 & 28769 \\ 
				Wavy-hair & 11342 & 40640 \\
				\midrule
				Total & \multicolumn{2}{c}{162770} \\ 
				\bottomrule
			\end{tabular}
		\end{small}
	\end{subtable}
	\hfill
	\begin{subtable}[!t]{0.55\textwidth}
		\centering
		\begin{small}
			\begin{tabular}{ccc}
				\toprule
				\textbf{CivilComments-Wilds} & Toxic & Non-toxic \\
				\midrule
				Male  & 4437 & 25373 \\
				Female & 4962 & 31282 \\ 
				LGBTQ & 2265 & 6155 \\ 
				Christian & 2446 & 24292 \\ 
				Muslim & 3125 & 10829 \\ 
				Other Religions & 1003 & 5541 \\ 
				Black & 3111 & 6785 \\ 
				White & 4682 & 12016 \\
				\midrule
				Total & \multicolumn{2}{c}{269038} \\ 
				\bottomrule
			\end{tabular}
		\end{small}
	\end{subtable}
	
\end{table}

Table \ref{tab:celeba-domains} lists the number of training instances in each domain of each dataset. Each instance may belong to zero, one or more domains. In CivilComments-Wilds, the aggregated group size of the 16 groups is less than the total number 269,038, because most online comments do not contain sensitive words.

\subsection{Model Selection}
\label{app:model-select}

\begin{table}[!t]
	\caption{The average and worst-case test accuracies of the best models selected by different strategies. (\%)}
	\label{tab:model-selection}
	\begin{center}
		\vskip 0.15in
		\begin{small}
			\begin{tabular}{ cccc } 
				\toprule
				\textbf{Training Algorithm} & \textbf{Model Selection} & \textbf{Average Accuracy}  & \textbf{Worst-case Accuracy}\\ 
				\midrule
				\multirow{4}{*}{ERM} & Oracle & $95.01 \pm 0.38$ & $53.94 \pm 2.02$ \\ 
				& Max Avg Acc & $95.65 \pm 0.05$ & $45.83 \pm 1.87$ \\ 
				& Min CVaR & $95.68 \pm 0.04$ & $44.83 \pm 2.74$ \\ 
				& Min CVaR-DORO & $95.69 \pm 0.04$ & $44.50 \pm 2.72$ \\ 
				\midrule
				& Oracle & $95.52 \pm 0.08$ & $49.94 \pm 3.36$ \\ 
				CVaR & Max Avg Acc & $95.74 \pm 0.06$ & $39.28 \pm 3.58$ \\ 
				($\alpha=0.2$) & Min CVaR & $95.79 \pm 0.05$ & $38.67 \pm 2.06$ \\ 
				& Min CVaR-DORO & $95.81 \pm 0.05$ & $38.83 \pm 2.05$ \\ 
				\midrule
				& Oracle & $92.91 \pm 0.48$ & $72.17 \pm 3.14$ \\ 
				CVaR-DORO& Max Avg Acc & $95.60 \pm 0.05$ & $45.39 \pm 3.22$ \\ 
				($\alpha = 0.2$, $\epsilon = 0.005$)& Min CVaR & $95.58 \pm 0.06$ & $39.83 \pm 2.37$ \\ 
				& Min CVaR-DORO & $95.56 \pm 0.07$ & $41.28 \pm 3.26$ \\ 
				\midrule
				& Oracle & $82.44 \pm 1.22$ & $63.36 \pm 2.51$ \\ 
				$\chi^2$-DRO & Max Avg Acc & $90.70 \pm 0.26$ & $20.67 \pm 3.86$ \\ 
				($\alpha=0.2$) & Min CVaR & $87.28 \pm 2.05$ & $21.44 \pm 11.13$ \\ 
				& Min CVaR-DORO & $89.16 \pm 1.41$ & $25.50 \pm 9.14$ \\ 
				\midrule
				& Oracle & $80.73 \pm 1.41$ & $65.36 \pm 1.02$ \\ 
				$\chi^2$-DORO& Max Avg Acc & $90.06 \pm 0.57$ & $22.06 \pm 5.82$ \\ 
				($\alpha = 0.2$, $\epsilon = 0.005$)& Min CVaR & $84.37 \pm 4.08$ & $29.83 \pm 12.10$ \\ 
				& Min CVaR-DORO & $88.76 \pm 0.81$ & $23.61 \pm 7.45$ \\ 
				\bottomrule
			\end{tabular}
		\end{small}
	\end{center}
\end{table}

In Section \ref{sec:experiments} we assume access to a domain-aware validation set, which is not available in real domain-oblivious tasks. In this part we study several domain-oblivious model selection strategies, and discuss why model selection is hard.

We study the following model selection strategies:
\begin{itemize}
	\item Max Average Accuracy: The model with the highest average accuracy in validation.
	\item Min CVaR: The model with the lowest CVaR risk ($\alpha=0.2$) over the validation set.
	\item Min CVaR-DORO: The model with the lowest CVaR-DORO risk ($\alpha=0.2, \epsilon = 0.005$) over the validation set.
\end{itemize}

Note that selecting the model that achieves the highest average accuracy over the worst $\alpha$ portion of the data is almost equivalent to the Max Average Accuracy strategy because the model with the highest average accuracy over the population also achieves the highest accuracy on the worst $\alpha$ portion (see e.g. \cite{hu2018does}, Theorem 1).

We conduct experiments on CelebA and report the results in Table \ref{tab:model-selection}. From the table we draw the following conclusions:
\begin{enumerate}
    \item For every training algorithm, the oracle strategy achieves a much higher worst-case test accuracy than the other three strategies, and the gap between the oracle and the non-oracle strategies for DRO and DORO is larger than ERM. While it is expected that the oracle achieves a higher worst-case accuracy, the large gap indicates that there is still huge room for improvement.
    \item For $\chi^2$-DRO/DORO, Min CVaR and Min CVaR-DORO work better than Max Average Accuracy. However, for the other three algorithms, Max Average Accuracy is better. This shows that model selection based on CVaR and selection based on CVaR-DORO are not good strategies. 
    \item With the three non-oracle strategies, ERM achieves the highest worst-case test accuracy. This does not mean that DRO and DORO are not as good as ERM, but suggests that we need other model selection strategies that work better with DRO and DORO.
\end{enumerate}

The reason why Min CVaR is not a good strategy is that CVaR does not decrease monotonically with $\gR_{\max}$. Corollary \ref{prop:dro-fwod} only guarantees that CVaR is an upper bound of $\gR_{\max}$, but the $\theta$ that achieves the minimum CVaR does not necessarily have the smallest $\gR_{\max}$. For the same reason, Min CVaR-DORO is not a good strategy either. 

Model selection under the domain oblivious setting is a very difficult task. In fact, Theorem 1 of \cite{hu2018does} implies that no strategy can be provably better than Max Average Accuracy under the domain-oblivious setting, i.e. for any model selection strategy, there always exist $\gD_1,\cdots,\gD_K$ such that the model it selects is not better than the model selected by the Max Average Accuracy strategy. Thus, to design a provably model selection strategy, prior knowledge or reasonable assumptions on the domains are necessary.

\subsection{Training Hyperparameters}
\label{app:best-param}

On the COMPAS dataset, we use a two-layer feed-forward neural network activated by ReLU as the classification model. For optimization we use ASGD with learning rate 0.01. The batch size is 128. The hyperparameters we used in Table \ref{tab:results-acc} were: $\alpha=0.5$ for CVaR; $\alpha=0.5$, $\epsilon=0.2$ for CVaR-DORO; $\alpha=0.5$ for $\chi^2$-DRO; $\alpha=0.5$, $\epsilon=0.2$ for $\chi^2$-DORO.

On the CelebA dataset, we use a standard ResNet18 as the classification model. For optimization we use momentum SGD with learning rate 0.001, momentum 0.9 and weight decay 0.001. The batch size is 400. The hyperparameters we used in Table \ref{tab:results-acc} were: $\alpha=0.1$ for CVaR; $\alpha=0.2$, $\epsilon=0.005$ for CVaR-DORO; $\alpha=0.25$ for $\chi^2$-DRO; $\alpha=0.25$, $\epsilon=0.01$ for $\chi^2$-DORO.

On the CivilComments-Wilds dataset, we use a pretrained BERT-base-uncased model as the classification model. For optimization, we use AdamW with learning rate 0.00001 and weight decay 0.01. The batch size is 128. The hyperparameters we used in Table \ref{tab:results-acc} were: $\alpha=0.1$ for CVaR; $\alpha=0.1$, $\epsilon=0.01$ for CVaR-DORO; $\alpha=0.2$ for $\chi^2$-DRO; $\alpha=0.2$, $\epsilon=0.01$ for $\chi^2$-DORO.


\begin{thebibliography}{47}
\providecommand{\natexlab}[1]{#1}
\providecommand{\url}[1]{\texttt{#1}}
\expandafter\ifx\csname urlstyle\endcsname\relax
  \providecommand{\doi}[1]{doi: #1}\else
  \providecommand{\doi}{doi: \begingroup \urlstyle{rm}\Url}\fi

\bibitem[Barocas \& Selbst(2016)Barocas and Selbst]{barocas2016big}
Barocas, S. and Selbst, A.~D.
\newblock Big data's disparate impact.
\newblock \emph{Calif. L. Rev.}, 104:\penalty0 671, 2016.

\bibitem[Bickel et~al.(2007)Bickel, Br{\"u}ckner, and
  Scheffer]{bickel2007discriminative}
Bickel, S., Br{\"u}ckner, M., and Scheffer, T.
\newblock Discriminative learning for differing training and test
  distributions.
\newblock In \emph{Proceedings of the 24th international conference on Machine
  learning}, pp.\  81--88, 2007.

\bibitem[Borkan et~al.(2019)Borkan, Dixon, Sorensen, Thain, and
  Vasserman]{borkan2019nuanced}
Borkan, D., Dixon, L., Sorensen, J., Thain, N., and Vasserman, L.
\newblock Nuanced metrics for measuring unintended bias with real data for text
  classification.
\newblock In \emph{Companion Proceedings of The 2019 World Wide Web
  Conference}, pp.\  491--500, 2019.

\bibitem[Brent(1971)]{brent1971algorithm}
Brent, R.~P.
\newblock An algorithm with guaranteed convergence for finding a zero of a
  function.
\newblock \emph{The Computer Journal}, 14\penalty0 (4):\penalty0 422--425,
  1971.

\bibitem[Cressie \& Read(1984)Cressie and Read]{cressie1984multinomial}
Cressie, N. and Read, T.~R.
\newblock Multinomial goodness-of-fit tests.
\newblock \emph{Journal of the Royal Statistical Society: Series B
  (Methodological)}, 46\penalty0 (3):\penalty0 440--464, 1984.

\bibitem[Devlin et~al.(2019)Devlin, Chang, Lee, and
  Toutanova]{devlin-etal-2019-bert}
Devlin, J., Chang, M.-W., Lee, K., and Toutanova, K.
\newblock {BERT}: Pre-training of deep bidirectional transformers for language
  understanding.
\newblock In \emph{Proceedings of the 2019 Conference of the North {A}merican
  Chapter of the Association for Computational Linguistics: Human Language
  Technologies, Volume 1 (Long and Short Papers)}, pp.\  4171--4186,
  Minneapolis, Minnesota, June 2019. Association for Computational Linguistics.
\newblock \doi{10.18653/v1/N19-1423}.

\bibitem[Diakonikolas et~al.(2017)Diakonikolas, Kamath, Kane, Li, Moitra, and
  Stewart]{diakonikolas2017being}
Diakonikolas, I., Kamath, G., Kane, D.~M., Li, J., Moitra, A., and Stewart, A.
\newblock Being robust (in high dimensions) can be practical.
\newblock In Precup, D. and Teh, Y.~W. (eds.), \emph{Proceedings of the 34th
  International Conference on Machine Learning}, volume~70 of \emph{Proceedings
  of Machine Learning Research}, pp.\  999--1008, International Convention
  Centre, Sydney, Australia, 06--11 Aug 2017.

\bibitem[Diakonikolas et~al.(2019)Diakonikolas, Kamath, Kane, Li, Moitra, and
  Stewart]{diakonikolas2019robust}
Diakonikolas, I., Kamath, G., Kane, D., Li, J., Moitra, A., and Stewart, A.
\newblock Robust estimators in high-dimensions without the computational
  intractability.
\newblock \emph{SIAM Journal on Computing}, 48\penalty0 (2):\penalty0 742--864,
  2019.

\bibitem[Duchi \& Namkoong(2018)Duchi and Namkoong]{duchi2018learning}
Duchi, J. and Namkoong, H.
\newblock Learning models with uniform performance via distributionally robust
  optimization.
\newblock \emph{arXiv preprint arXiv:1810.08750}, 2018.

\bibitem[Dwork et~al.(2012)Dwork, Hardt, Pitassi, Reingold, and
  Zemel]{dwork2012fairness}
Dwork, C., Hardt, M., Pitassi, T., Reingold, O., and Zemel, R.
\newblock Fairness through awareness.
\newblock In \emph{Proceedings of the 3rd innovations in theoretical computer
  science conference}, pp.\  214--226, 2012.

\bibitem[Galar et~al.(2011)Galar, Fernandez, Barrenechea, Bustince, and
  Herrera]{galar2011review}
Galar, M., Fernandez, A., Barrenechea, E., Bustince, H., and Herrera, F.
\newblock A review on ensembles for the class imbalance problem: bagging-,
  boosting-, and hybrid-based approaches.
\newblock \emph{IEEE Transactions on Systems, Man, and Cybernetics, Part C
  (Applications and Reviews)}, 42\penalty0 (4):\penalty0 463--484, 2011.

\bibitem[Gulrajani \& Lopez-Paz(2021)Gulrajani and Lopez-Paz]{gulrajani2021in}
Gulrajani, I. and Lopez-Paz, D.
\newblock In search of lost domain generalization.
\newblock In \emph{International Conference on Learning Representations}, 2021.

\bibitem[Hardt et~al.(2016)Hardt, Price, Price, and Srebro]{hardt2016equality}
Hardt, M., Price, E., Price, E., and Srebro, N.
\newblock Equality of opportunity in supervised learning.
\newblock In Lee, D., Sugiyama, M., Luxburg, U., Guyon, I., and Garnett, R.
  (eds.), \emph{Advances in Neural Information Processing Systems}, volume~29,
  pp.\  3315--3323. Curran Associates, Inc., 2016.

\bibitem[Hashimoto et~al.(2018)Hashimoto, Srivastava, Namkoong, and
  Liang]{pmlr-v80-hashimoto18a}
Hashimoto, T., Srivastava, M., Namkoong, H., and Liang, P.
\newblock Fairness without demographics in repeated loss minimization.
\newblock In Dy, J. and Krause, A. (eds.), \emph{International Conference on
  Machine Learning}, volume~80 of \emph{Proceedings of Machine Learning
  Research}, pp.\  1929--1938, Stockholmsmässan, Stockholm Sweden, 10--15 Jul
  2018. PMLR.

\bibitem[He et~al.(2016)He, Zhang, Ren, and Sun]{he2016deep}
He, K., Zhang, X., Ren, S., and Sun, J.
\newblock Deep residual learning for image recognition.
\newblock In \emph{Proceedings of the IEEE conference on computer vision and
  pattern recognition}, pp.\  770--778, 2016.

\bibitem[Hu et~al.(2018)Hu, Niu, Sato, and Sugiyama]{hu2018does}
Hu, W., Niu, G., Sato, I., and Sugiyama, M.
\newblock Does distributionally robust supervised learning give robust
  classifiers?
\newblock In \emph{International Conference on Machine Learning}, pp.\
  2029--2037. PMLR, 2018.

\bibitem[Huang et~al.(2006)Huang, Gretton, Borgwardt, Sch{\"o}lkopf, and
  Smola]{huang2006correcting}
Huang, J., Gretton, A., Borgwardt, K., Sch{\"o}lkopf, B., and Smola, A.
\newblock Correcting sample selection bias by unlabeled data.
\newblock \emph{Advances in neural information processing systems},
  19:\penalty0 601--608, 2006.

\bibitem[Huber(1992)]{huber1992robust}
Huber, P.~J.
\newblock Robust estimation of a location parameter.
\newblock In \emph{Breakthroughs in statistics}, pp.\  492--518. Springer,
  1992.

\bibitem[Japkowicz(2000)]{japkowicz2000class}
Japkowicz, N.
\newblock The class imbalance problem: Significance and strategies.
\newblock In \emph{Proc. of the Int’l Conf. on Artificial Intelligence},
  volume~56. Citeseer, 2000.

\bibitem[Koh et~al.(2020)Koh, Sagawa, Marklund, Xie, Zhang, Balsubramani, Hu,
  Yasunaga, Phillips, Beery, et~al.]{koh2020wilds}
Koh, P.~W., Sagawa, S., Marklund, H., Xie, S.~M., Zhang, M., Balsubramani, A.,
  Hu, W., Yasunaga, M., Phillips, R.~L., Beery, S., et~al.
\newblock Wilds: A benchmark of in-the-wild distribution shifts.
\newblock \emph{arXiv preprint arXiv:2012.07421}, 2020.

\bibitem[Kothari et~al.(2018)Kothari, Steinhardt, and
  Steurer]{DBLP:conf/stoc/KothariSS18}
Kothari, P.~K., Steinhardt, J., and Steurer, D.
\newblock Robust moment estimation and improved clustering via sum of squares.
\newblock In Diakonikolas, I., Kempe, D., and Henzinger, M. (eds.),
  \emph{Proceedings of the 50th Annual {ACM} {SIGACT} Symposium on Theory of
  Computing, {STOC} 2018, Los Angeles, CA, USA, June 25-29, 2018}, pp.\
  1035--1046. {ACM}, 2018.

\bibitem[Kusner et~al.(2017)Kusner, Loftus, Russell, and
  Silva]{kusner2017counterfactual}
Kusner, M.~J., Loftus, J., Russell, C., and Silva, R.
\newblock Counterfactual fairness.
\newblock In \emph{Advances in neural information processing systems}, pp.\
  4066--4076, 2017.

\bibitem[Lahoti et~al.(2020)Lahoti, Beutel, Chen, Lee, Prost, Thain, Wang, and
  Chi]{lahoti2020fairness}
Lahoti, P., Beutel, A., Chen, J., Lee, K., Prost, F., Thain, N., Wang, X., and
  Chi, E.
\newblock Fairness without demographics through adversarially reweighted
  learning.
\newblock \emph{Advances in Neural Information Processing Systems}, 33, 2020.

\bibitem[Lai et~al.(2016)Lai, Rao, and Vempala]{lai2016agnostic}
Lai, K.~A., Rao, A.~B., and Vempala, S.
\newblock Agnostic estimation of mean and covariance.
\newblock In \emph{2016 IEEE 57th Annual Symposium on Foundations of Computer
  Science (FOCS)}, pp.\  665--674. IEEE, 2016.

\bibitem[Larson et~al.(2016)Larson, Mattu, Kirchner, and Angwin]{larson2016we}
Larson, J., Mattu, S., Kirchner, L., and Angwin, J.
\newblock How we analyzed the compas recidivism algorithm.
\newblock \emph{ProPublica (5 2016)}, 9\penalty0 (1), 2016.

\bibitem[Lee et~al.(2020)Lee, Park, and Shin]{NEURIPS2020_9f60ab2b}
Lee, J., Park, S., and Shin, J.
\newblock Learning bounds for risk-sensitive learning.
\newblock In Larochelle, H., Ranzato, M., Hadsell, R., Balcan, M.~F., and Lin,
  H. (eds.), \emph{Advances in Neural Information Processing Systems},
  volume~33, pp.\  13867--13879. Curran Associates, Inc., 2020.

\bibitem[Liu et~al.(2015)Liu, Luo, Wang, and Tang]{liu2015deep}
Liu, Z., Luo, P., Wang, X., and Tang, X.
\newblock Deep learning face attributes in the wild.
\newblock In \emph{Proceedings of the IEEE international conference on computer
  vision}, pp.\  3730--3738, 2015.

\bibitem[Michel et~al.(2021)Michel, Hashimoto, and Neubig]{michel2021modeling}
Michel, P., Hashimoto, T., and Neubig, G.
\newblock Modeling the second player in distributionally robust optimization.
\newblock In \emph{International Conference on Learning Representations}, 2021.

\bibitem[Namkoong \& Duchi(2016)Namkoong and Duchi]{namkoong2016stochastic}
Namkoong, H. and Duchi, J.~C.
\newblock Stochastic gradient methods for distributionally robust optimization
  with f-divergences.
\newblock In \emph{Advances in neural information processing systems}, pp.\
  2208--2216, 2016.

\bibitem[Oren et~al.(2019)Oren, Sagawa, Hashimoto, and
  Liang]{oren-etal-2019-distributionally}
Oren, Y., Sagawa, S., Hashimoto, T., and Liang, P.
\newblock Distributionally robust language modeling.
\newblock In \emph{Proceedings of the 2019 Conference on Empirical Methods in
  Natural Language Processing and the 9th International Joint Conference on
  Natural Language Processing (EMNLP-IJCNLP)}, pp.\  4227--4237, Hong Kong,
  China, November 2019. Association for Computational Linguistics.
\newblock \doi{10.18653/v1/D19-1432}.

\bibitem[Pan \& Yang(2009)Pan and Yang]{pan2009survey}
Pan, S.~J. and Yang, Q.
\newblock A survey on transfer learning.
\newblock \emph{IEEE Transactions on knowledge and data engineering},
  22\penalty0 (10):\penalty0 1345--1359, 2009.

\bibitem[Patel et~al.(2015)Patel, Gopalan, Li, and Chellappa]{patel2015visual}
Patel, V.~M., Gopalan, R., Li, R., and Chellappa, R.
\newblock Visual domain adaptation: A survey of recent advances.
\newblock \emph{IEEE signal processing magazine}, 32\penalty0 (3):\penalty0
  53--69, 2015.

\bibitem[Prasad et~al.(2018)Prasad, Suggala, Balakrishnan, and
  Ravikumar]{prasad2018robust}
Prasad, A., Suggala, A.~S., Balakrishnan, S., and Ravikumar, P.
\newblock Robust estimation via robust gradient estimation.
\newblock \emph{arXiv preprint arXiv:1802.06485}, 2018.

\bibitem[Prasad et~al.(2020)Prasad, Balakrishnan, and
  Ravikumar]{pmlr-v108-prasad20a}
Prasad, A., Balakrishnan, S., and Ravikumar, P.
\newblock A robust univariate mean estimator is all you need.
\newblock In Chiappa, S. and Calandra, R. (eds.), \emph{Proceedings of the
  Twenty Third International Conference on Artificial Intelligence and
  Statistics}, volume 108 of \emph{Proceedings of Machine Learning Research},
  pp.\  4034--4044. PMLR, 26--28 Aug 2020.

\bibitem[Quionero-Candela et~al.(2009)Quionero-Candela, Sugiyama, Schwaighofer,
  and Lawrence]{quionero2009dataset}
Quionero-Candela, J., Sugiyama, M., Schwaighofer, A., and Lawrence, N.~D.
\newblock \emph{Dataset shift in machine learning}.
\newblock The MIT Press, 2009.

\bibitem[Rawls(2001)]{rawls2001justice}
Rawls, J.
\newblock \emph{Justice as fairness: A restatement}.
\newblock Harvard University Press, 2001.

\bibitem[Sagawa et~al.(2020{\natexlab{a}})Sagawa, Koh, Hashimoto, and
  Liang]{sagawa2019distributionally}
Sagawa, S., Koh, P.~W., Hashimoto, T.~B., and Liang, P.
\newblock Distributionally robust neural networks for group shifts: On the
  importance of regularization for worst-case generalization.
\newblock In \emph{International Conference on Learning Representations},
  2020{\natexlab{a}}.

\bibitem[Sagawa et~al.(2020{\natexlab{b}})Sagawa, Raghunathan, Koh, and
  Liang]{sagawa2020investigation}
Sagawa, S., Raghunathan, A., Koh, P.~W., and Liang, P.
\newblock An investigation of why overparameterization exacerbates spurious
  correlations.
\newblock In III, H.~D. and Singh, A. (eds.), \emph{Proceedings of the 37th
  International Conference on Machine Learning}, volume 119 of
  \emph{Proceedings of Machine Learning Research}, pp.\  8346--8356. PMLR,
  13--18 Jul 2020{\natexlab{b}}.

\bibitem[Shen \& Sanghavi(2019)Shen and Sanghavi]{shen2019learning}
Shen, Y. and Sanghavi, S.
\newblock Learning with bad training data via iterative trimmed loss
  minimization.
\newblock In \emph{International Conference on Machine Learning}, pp.\
  5739--5748. PMLR, 2019.

\bibitem[Shimodaira(2000)]{shimodaira2000improving}
Shimodaira, H.
\newblock Improving predictive inference under covariate shift by weighting the
  log-likelihood function.
\newblock \emph{Journal of statistical planning and inference}, 90\penalty0
  (2):\penalty0 227--244, 2000.

\bibitem[Tan et~al.(2018)Tan, Sun, Kong, Zhang, Yang, and Liu]{tan2018survey}
Tan, C., Sun, F., Kong, T., Zhang, W., Yang, C., and Liu, C.
\newblock A survey on deep transfer learning.
\newblock In \emph{International conference on artificial neural networks},
  pp.\  270--279. Springer, 2018.

\bibitem[Tukey(1960)]{tukey1960survey}
Tukey, J.~W.
\newblock A survey of sampling from contaminated distributions.
\newblock \emph{Contributions to probability and statistics}, pp.\  448--485,
  1960.

\bibitem[Wang \& Deng(2018)Wang and Deng]{wang2018deep}
Wang, M. and Deng, W.
\newblock Deep visual domain adaptation: A survey.
\newblock \emph{Neurocomputing}, 312:\penalty0 135--153, 2018.

\bibitem[Xu et~al.(2020)Xu, Dan, Khim, and Ravikumar]{xu2020class}
Xu, Z., Dan, C., Khim, J., and Ravikumar, P.
\newblock Class-weighted classification: Trade-offs and robust approaches.
\newblock In III, H.~D. and Singh, A. (eds.), \emph{Proceedings of the 37th
  International Conference on Machine Learning}, volume 119 of
  \emph{Proceedings of Machine Learning Research}, pp.\  10544--10554. PMLR,
  13--18 Jul 2020.

\bibitem[Zafar et~al.(2017)Zafar, Valera, Gomez~Rodriguez, and
  Gummadi]{zafar2017fairness}
Zafar, M.~B., Valera, I., Gomez~Rodriguez, M., and Gummadi, K.~P.
\newblock Fairness beyond disparate treatment \& disparate impact: Learning
  classification without disparate mistreatment.
\newblock In \emph{Proceedings of the 26th international conference on world
  wide web}, pp.\  1171--1180, 2017.

\bibitem[Zemel et~al.(2013)Zemel, Wu, Swersky, Pitassi, and
  Dwork]{zemel2013learning}
Zemel, R., Wu, Y., Swersky, K., Pitassi, T., and Dwork, C.
\newblock Learning fair representations.
\newblock In \emph{International Conference on Machine Learning}, pp.\
  325--333, 2013.

\bibitem[Zhu et~al.(2020)Zhu, Jiao, and Steinhardt]{zhu2020generalized}
Zhu, B., Jiao, J., and Steinhardt, J.
\newblock Generalized resilience and robust statistics, 2020.

\end{thebibliography}
\end{document}